\documentclass[manuscript,screen]{acmart}

\AtBeginDocument{%
  \providecommand\BibTeX{{%
    \normalfont B\kern-0.5em{\scshape i\kern-0.25em b}\kern-0.8em\TeX}}}

\setcopyright{acmcopyright}
\copyrightyear{2020}
\acmYear{2020}
\acmDOI{10.1145/1122445.1122456}

\usepackage{amsmath}
\usepackage{amsfonts}
\usepackage{ifthen}
\usepackage{color,soul}
\usepackage[linesnumbered]{algorithm2e}
\usepackage{url}
\usepackage{tikz}
\usetikzlibrary{arrows,backgrounds,decorations,decorations.pathmorphing,
	positioning,fit,automata,shapes,snakes,patterns,plotmarks,calc,trees}
\usepackage{graphicx}
\usepackage{latexsym}	
\usepackage{alltt}
\usepackage[obeyFinal]{todonotes}
\usepackage{wrapfig}
\usepackage{subcaption}
\usepackage[per-mode=symbol]{siunitx}
\usepackage{paralist}
\usepackage{color,colortbl}
\usepackage{pgfgantt}

\usepackage{epsfig,color}
\usepackage{algpseudocode}
\usepackage{mdwlist}
\usepackage{setspace} 
\usepackage{hyperref}
\usepackage[depth=1]{bookmark}
\usepackage{listings}
\usepackage{caption}
\usepackage{xspace}
\usepackage{tikz}
\usepackage{stmaryrd}
\usepackage{multirow}
\usepackage{multicol}

\usepackage{ltl}
\usepackage{habbas-macros}
\def \TenseLogic{\CTLs}
\newcommand{\atom}[1]{\sffamily{#1}}
\newcommand{\matom}[1]{\mathsf{#1}}

\begin{document}
\title[Algorithmic Ethics]{Algorithmic Ethics: Formalization and Verification of Autonomous Vehicle Obligations}

\author{Colin Shea-Blymyer}
\email{sheablyc@oregonstate.edu}

\author{Houssam Abbas}
\email{houssam.abbas@oregonstate.edu}
\affiliation{%
  \institution{Oregon State University}
}

\renewcommand{\shortauthors}{Shea-Blymyer and Abbas}

\begin{abstract}
  We develop a formal framework for automatic reasoning about the obligations of autonomous cyber-physical systems, including their social and ethical obligations.
  Obligations, permissions and prohibitions are distinct from a system's mission, and are a necessary part of specifying advanced, adaptive AI-equipped systems. 
  They need a dedicated deontic logic of obligations to formalize them.
  \yhl{Most existing deontic logics lack corresponding algorithms and system models that permit automatic verification.}
  We demonstrate how a particular deontic logic, Dominance Act Utilitarianism (DAU)~\cite{Horty01DLAgency}, is a suitable starting point for formalizing the obligations of autonomous systems like self-driving cars.
  We demonstrate its usefulness by formalizing a subset of Responsibility-Sensitive Safety (RSS) in DAU; RSS is an industrial proposal for how self-driving cars should and should not behave in traffic.
  We show that certain logical consequences of RSS are undesirable, indicating a need to further refine the proposal.
  We also demonstrate how obligations can change over time, which is necessary for long-term autonomy. 
  We then demonstrate a model-checking algorithm for DAU formulas on weighted transition systems, 
  and illustrate it by model-checking obligations of a self-driving car controller from the literature.
\end{abstract}


\setcopyright{acmcopyright}
\acmJournal{TCPS}
\acmYear{2021} \acmVolume{1} \acmNumber{1} \acmArticle{1} \acmMonth{1} \acmPrice{15.00}\acmDOI{10.1145/3460975}



\maketitle

\section{Introduction}
\label{sec:intro}
The need for embodied Cyber-Physical Systems (CPS) that are fully autonomous, update their own objectives and interact with us in our daily lives is more obvious today than ever. 
To cite one example, the Covid-19 pandemic has highlighted the need for nursing robots that can check on patients in high-risk situations, self-driving vehicles that deliver essential goods to people who cannot get them, and companion robots that understand and adapt to different living situations like those of elderly people or incapacitated persons.
We refer to these different types of systems as {\itshape human-scale CPS} : embodied CPS that interact with humans and their environment, and are perceived as being reasonably intelligent and autonomous. 
The common thread to all of these applications is that the autonomous CPS is seen as just another agent in our environment, and our interactions with it assume a wide range of social expectations built through our interactions with other humans. 
Indeed, the success of these systems depends on their respect for these social norms of interaction, and more particularly on the robot's respect for \textit{ethical guidelines} that are as necessary as they are ambiguous. 
These obligations are distinct from the CPS' mission, which is, for example, to go from A to B without collisions.
Obligations place constraints on how the CPS achieves its mission, and might be violated.
Safety and performance are no longer sufficient criteria for a successful CPS design: ethical, and more generally social, obligations must be formalized, verified, and where possible, enforced. \yhl{In this work we tackle the challenges of formalizing ethical obligations in a useful and interpretable way, analyzing the properties of these obligations, and automatically verifying that a system has given obligations.}

In the fields of Artificial Intelligence (AI) and Logic, the formalization of ethical and social obligations dates back at least to Mally~\cite{Mally26,dlMallyStanford}, with most of the focus going towards developing the `right' logics and simulation-based studies of normative systems. 
\yhl{While these logics are interpretable, they lack system models of the agents under obligation. In \cite{Gerdes2015}, limited ethical requirements are modeled in the costs and/or constraints of a classical optimal control problem. Precisely defining the costs and constraints that embed ethics into a control problem is challenging, and providing a high-level interpretation of what a particular choice logically entails is generally not possible. E.g. how does the behavior change qualitatively if a slack variable is increased or a weight is decreased?}
In CPS, The formalization, verification, and enforcement of ethical and social obligations has not been adequately tackled. 
This paper develops a formal framework and tool for the analysis of the ethical obligations of human-scale CPS, with applications in self-driving cars.

The formal verification and control of CPS safety and performance has relied on \textit{alethic temporal logics}, like Linear Temporal Logic~\cite{Pnueli77sfcs}, to express behavioral specifications of system models.
Alethic logic is the logic of \textit{necessity and possibility}:
for example, if $p$ is a predicate, $\always p$ says that $p$ is true in every accessible world - that is, $p$ is necessary.
Possibility is then formalized as $\eventually p \defeq \neg \always \neg p$: saying that $p$ is possible is the same as saying that it is not the case that $\neg p$ is necessary.
And so on.
The best known instantiation of this in Verification is LTL~\cite{MannaP92}, in which an accessible world is a moment in the linear future.
Thus $\always p$ formalizes `$p$ is true in every future moment', and $\eventually p$ formalizes `$p$ is true in some future moment'.
It is natural to want to leverage alethic logics and associated tools to formalize and study CPS obligations as well. 
However, it has been understood for over 70 years that \textit{the logic of obligations is different from that of necessity}~\cite{dlStanford}: applying alethic logic rules to obligation statements can lead to conclusions that are intuitively paradoxical or undesirable.
Consider the following statements:
\begin{enumerate}[A.]
	\item The car \textit{will} eventually change lanes: this is a statement about \yhl{necessity}. It says nothing about whether the car plays an active role in the lane change (e.g., perhaps it will hit a slippery patch), or whether it \textit{should} change lanes.
	\item The car \textit{can} change lanes: this is a statement about ability. The car might be able to do something, but does not \yhl{actually} do it.
	\item The car \textit{sees to it that} it changes lanes: this is a statement about agency. It tells us that the car ensures that it changes lanes. \footnote{The `see to it' phraseology is very common in Logic and we use it in this paper.} 
	I.e., it is an active agent in the lane change.	
	\item The car \textit{ought} to change lanes: this is a statement about obligation. The car, for example, might fail to meet its obligation, either by choice or because it can't change lanes.
\end{enumerate}
These are qualitatively different statements and there is no a priori equivalence between any two of them.
The logic we adopt should reflect this: its operators and inference rules should model these aspects \textit{in the logic}, without having to add new atomic propositions for every new concept and situation that occurs to us.
Alethic logics like LTL cannot do so.\footnote{Anderson and Kanger attempted a reduction of obligations to alethic logic. See~\cite[Section 3]{dlStanford} for a discussion.}

We now give a simple but fundamental example, drawn from~\cite{dlStanford}, illustrating this inability of alethic logic.
%
One might be tempted to formalize obligation using the necessity operator $\always$: that is, formalize `The car ought to change lane' by $\always \texttt{change-lane}$.
However, in alethic logic, $\always p \implies p$: if $p$ is necessarily true then it is true.
If we use $\always$ for obligation this reads as $\mathbf{Obligatory}~p \implies p$: this inference is clearly unacceptable because agents sometimes violate their obligations so some obligatory things are not true.
This leads to the question of what an agent ought to do when some primary obligations are violated. 
I.e. the study of statements of the form $\mathbf{Obligatory}~p \land \neg p \implies ...$.
This is not possible if obligation is formalized using $\always$ in pure alethic logic, since $\always p \land \neg p \implies q$ is trivially true for \textit{any} $p$ and $q$.

\textit{Deontic logic}~\cite{DLHandbook} has been developed specifically to reason about obligations, starting with von Wright~\cite{vonWright51DL}. 
It is used in contract law, including software contracts, and is an active area of research in Logic-based AI~\cite{KulickiTMDeon18}.
There are many flavors of deontic logic~\cite{McNamaraChapter}.
In this paper, we adopt the logic of Dominance Act Utilitarianism (DAU) developed by Horty~\cite{Horty01DLAgency} because it explicitly models all four aspects above: necessity, agency, ability and obligation.
We first extend DAU to formalize the obligations of human-scale CPS with complex missions.
We then formalize a subset of Intel's Responsibility-Sensitive Safety, or RSS, in DAU~\cite{RSSv6}. 
RSS proposes a set of rules to be followed by self-driving cars to avoid collisions.
To promote `naturalistic driving', RSS places an \textit{obligation} to avoid aggressive driving while giving \textit{permission} to drive assertively.
Using our DAU formalization of RSS, we demonstrate that RSS allows a car to facilitate an accident in traffic, clearly an undesirable position; this points to the need to further refine the RSS proposal. 

We develop the first model-checking algorithm for DAU formulas, to determine whether a system model has a given obligation or not. 
We implemented the model-checker and present results on a self-driving car controller.
An obligation constitutes a constraint on the CPS controller, and can be integrated into the controller's objective;
thus designing obligations and checking them is conceptually akin to reward shaping in Reinforcement Learning~\cite{Wiewiora10RewardShaping}.

When studying an autonomous CPS' obligations, it is also necessary to analyze how these obligations change over time, as a result of the agent's choices~\cite{Brunel08Propagation}.
For example if I ought to visit an ill relative today or tomorrow, and I don't visit them today, then it's reasonable to say that tomorrow, my residual obligation is to visit them. 
It is important that the formal conclusions yielded by the logic match such intuitive conclusions, in order to build trust in human-scale CPS.
We prove obligation propagation patterns for obligations expressed in DAU.

Our contributions in this paper are to\footnote{A preliminary conference version of this work appeared in~\cite{SheaBlymyerHSCC20}. This paper adds the analysis of temporal propagation (\yhl{item 3 above}), improves the RSS formalization significantly and formalizes assertive driving \yhl{--- a model of a social permission} (\yhl{in item 2}), adds a model-checking algorithm for conditional obligations, \yhl{as the original can not find histories that satisfy a condition} (\yhl{in item 4}) and implements both model-checkers and demonstrates their use (\yhl{item 5}).}:
\begin{enumerate}
	\item formalize the obligations of RSS in DAU, and highlight the subtle decisions that need to be made when developing a rigorous specification;
	\item derive undesirable consequences of the RSS obligations, pointing to the need for further refinements of RSS;
	\item demonstrate patterns for temporal propagation of obligations in DAU, \yhl{allowing evaluation of obligations inheritance};
	\item develop a model-checking algorithm of DAU specifications that allows us to establish whether a system has a given obligation or not; and
	\item implement the model-checker and demonstrate its use on a self-driving car from the litterature.
\end{enumerate}

\paragraph{Paper Organization.}
Section~\ref{sec:dau} defines DAU.
Section~\ref{sec:formalizing rss in dau} gives a first case study: the formalization of a subset of RSS rules in DAU, and some of their logical consequences.
Section~\ref{sec:obligation propagation} proves propagation patterns that hold in DAU.
Section~\ref{sec:mc dau} gives a model-checking algorithm for absolute and conditional DAU statements.
Section~\ref{sec:case-study-in-model-checking-ethical-decisions-in-self-driving-cars} demonstrates the use of the model-checker on a highway driving controller from the literature. 
Related work is reviewed in Section~\ref{sec:related work},
and Section~\ref{sec:conclusions} concludes the paper.

\section{Dominance Act Utilitarianism}
\label{sec:dau}
We adopt the logic of Dominance Act Utilitarianism (DAU) developed by Horty~\cite{Horty01DLAgency} because it explicitly models all four aspects listed in the Introduction: necessity, agency, ability and obligation.
It includes a temporal logic as a component so we can describe temporal behaviors essential to system design, 
and it uses branching time, essential for modeling uncontrollable environments.
It has an intrinsic computational structure which makes it appealing for CPS verification and control purposes: the agent's obligations are derived from maximizing utility, so DAU can be viewed as the deontic logic of utility maximization in non-deterministic systems.
As such, it gives a \textit{logical interpretation} to the behavior of systems that maximize utility, such as~\cite{Gerdes2015}.
This section summarizes the main aspects of DAU developed in~\cite{Horty01DLAgency}.

\paragraph{Syntax.}
Let $\Agent$ be a finite set of agents, which represent, for example, the cars in traffic.
A DAU formula is obtained as follows:
\begin{equation*}
A \defeq \formula~|~\neg A~|~A\land A~|~~\cstitaa~~|~~\dstit{\alpha}{A}~~|~~\Ostit{\alpha}{A}~~|~~\OstitC{\alpha}{A}{\formula}~~|~~\Next A
\end{equation*}
where $\alpha \in \Agent$, $\land,\neg$ are the usual boolean connectives, and $\formula$ is a formula in the logic \TenseLogic.
We use \TenseLogic~to specify the CPS' mission and to describe states of affairs in the world. 
We give the informal description of \TenseLogic~operators and refer the reader to~\cite{ClarkeGP99} for formal semantics: the temporal operator $\always$ means Always (now and in every future moment along this trace), $\eventually$ means Eventually (now or at some future moment along this trace), and $\release$ means Release: $\formula \release \psi$ means that either $\psi$ always holds, or it does not hold at some future moment and sometime before then $\formula$ holds.
The path quantifier $\forall$ means For all paths, and $\exists$ means There exists a path.
The DAU-specific operators informally mean the following:
$\cstitaa$ is the agency operator and says that $\alpha$ sees to it, or ensures, that $A$ is true;
$\dstit{\alpha}{A}$ is a variant on $\cstitaa$;
$\Ostit{\alpha}{A}$ is the obligation modality and says that $\alpha$ ought to ensure that $A$ is true;
finally, $\OstitC{\alpha}{A}{\formula}$ says that under the condition $\formula$, $\alpha$ ought to ensure that $A$ is true.
The rest of this section gives the formal semantics of these deontic operators.

\paragraph{Branching time.}
Let $\Tree$ be a set of \textit{moments} with an irreflexive, transitive ordering relation $<$ such that for any three moments $m_1,m_2,m_3$ in $\Tree$, if $m_1<m_3$ and $m_2 < m_3$ then either $m_1<m_2$ or $m_2 < m_1$.
There is a unique \textit{root moment} which we denote by 0.
A \textit{history} is a maximal linearly ordered set of moments from $\Tree$: intuitively, it is a branch of the tree that extends infinitely into the future.
Given a moment $m \in \Tree$, the set of histories that go through $m$ is $H_m \defeq  \{ h \such m \in h\}$.
See Fig.~\ref{fig:choices}.
We will frequently refer to moment/history pairs $m/h$, where $m \in$ \textit{Tree} and $h \in H_m$.
\begin{definition}\cite[Def. 2.2]{Horty01DLAgency}
	\label{def:frame and model}
	With $AP$ a set of atomic propositions,
	a \emph{branching time model} is a tuple 
	\begin{math}
	\Model = (\Tree, <,v)
	\end{math}
	where $\Tree$ is a tree of moments with ordering $<$ and 
	$v$ is a function that maps moments $m$ in $\Model$ to sets of atomic propositions from $2^{AP}$\yhl{, the set of subsets of $AP$}.\footnote{In the DAU formulation of~\cite{Horty01DLAgency}, $v$ maps $m/h$ pairs, rather than moments $m$, to subsets of $AP$. This is more general but disagrees with the common usage of atomic propositions in CPS Verification, so we adopt this more classical definition of $v$. The ideas of this paper are best explained without such (currently) unnecessary generalities.}
\end{definition}
%
\begin{figure*}[t]
	\centering
	\includegraphics[height=2.25in,keepaspectratio]{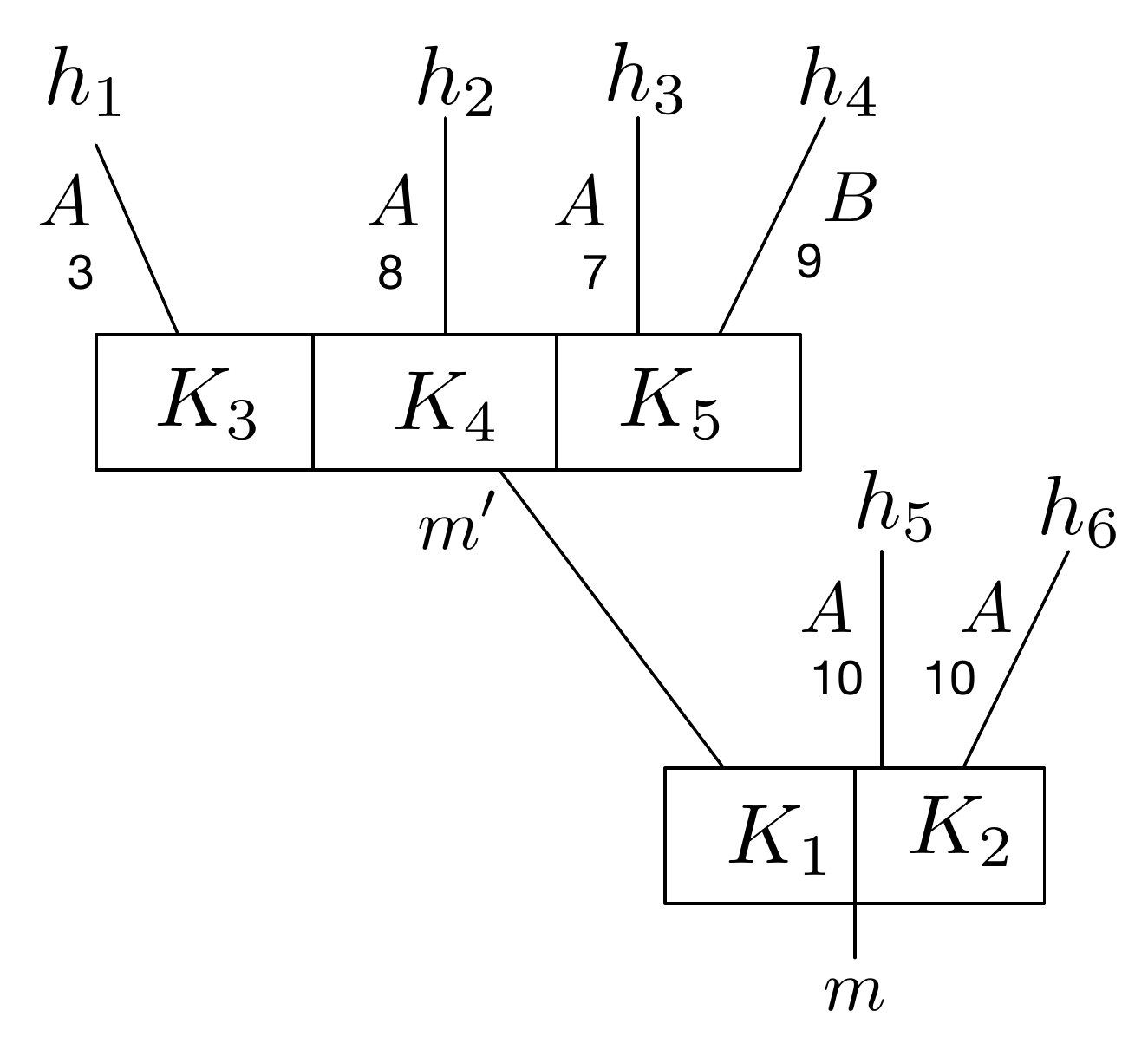}
	\caption{A utilitarian stit model for an agent $\alpha$ illustrating the main DAU definitions. Moments $m<m'$ with sets of histories $H_{m} = \{h_1,\ldots, h_6\}$ and $H_{m'}=\{h_1,\ldots, h_4\}$. 
		Each moment is marked with the actions available to $\alpha$ at that moment: $\Choiceam = \{K_1,K_2\}$ and $Choice_\alpha^{m'} = \{K_3,K_4,K_5\}$. 
	Action $K_2=\{h_5,h_6\}$ and $K_4=\{h_2\}$. 
	Each history is marked with the formula(s) that it satisfies at $m$ and with its value $\Value(h)$, e.g., $m/h_1$ satisfies $A$ and has value 3.
	$m/h_5 \models \cstitaa$ since $\Choiceam(h_5) = K_2$, and both $h_5$ and $h_6$ satisfy $A$. 
	On the other hand, $m/h_1 \not\models \cstit{\alpha}{A}$ since $\Choiceam(h_1) = K_1= \{h_1,h_2,h_3,h_4\}$ and $h_4$ does not satisfy $A$. 
	$\Optimalam = \{K_2\}$ so $m/h_5 \models \Ostit{\alpha}{A}$. 
	$\Optimal{\alpha}{m'} = \{K_4,K_5\}$ and so $\alpha$ has no obligations at $m'$ since there is no formula $\formula$ s.t. $|\formula|_{m'} \supseteq K_4 \cup K_5$ (See Def.~\ref{def:ought semantics}). 
	Finally, $m/h_5 \models \dstit{\alpha}{A}$ because $K_2 \subset |A|_m$ and $H_m \neq |A|_m=\{h_1,h_2,h_3,h_5,h_6\}$.} 
	\label{fig:choices}
\end{figure*}

In this paper, to retain a uniform satisfaction relation like~\cite{Horty01DLAgency}, we will speak of formulas holding or not at an $m/h$ pair and write $\Model, m/h \models \formula$, where it is always the case that $h\in H_m$.
When the formula is in \CTLs~there should be no confusion as a \CTLs~path formula is evaluated along $h$ and a state formula is evaluated at $m$.
Given a DAU statement $A$, the \textit{proposition} it expresses at moment $m$ is the set of histories where it holds starting at $m$
\begin{equation}
\label{eq:Amm}
|A|_m^\Model \defeq \{h \in H_m \such \Model, m/h \models A\}
\end{equation}
Where there's no risk of ambiguity, we drop $\Model$ from the notation, writing $|A|_m$, $m/h \models A$, etc.

\paragraph{Choice}
Consider an agent $\alpha \in \Agent$.
Formally, at $m$, an action $K$ is a subset of $H_m$: this is the subset of histories that are still realizable after taking the action.
At every moment $m$, $\alpha$ is faced with a choice of actions which we denote by $\Choiceam$.
So $\Choiceam \subset 2^{H_m}$.
See actions in Fig.~\ref{fig:choices}.
$\Choiceam$ must obey certain constraints given in the Supplementary material.
In what follows, $\Choiceam$ is assumed finite for every $\alpha$ and $m$. 

\paragraph{Agency}
Agency is defined via the `Chellas sees to it' operator $cstit$, named after Brian Chellas~\cite{chellas1968logical}.
Intuitively, an agent \textit{sees to it}, or ensures, that $A$ holds at $m/h$ if it takes an action $K$ s.t., whatever other history $h'$ could've resulted from $K$, $A$ is true at $m/h'$ as well. 
I.e., the non-determinism does not prevent $\alpha$ from guaranteeing $A$.
\begin{definition}[Chellas cstit]~\cite[Def. 2.7]{Horty01DLAgency}
	\label{def:cstit}
	With agent $\alpha$ and DAU statement $A$, let $\Choiceam(h)$ be the unique action that contains $h$. Then
	\[\Model,m/h \models [\alpha\,cstit:A] \text{ iff } \Choiceam(h) \subseteq |A|_m^\Model\]
	\end{definition}
If $K\subseteq |A|_m$ we say $K$ \textit{guarantees} $A$.
See Fig.~\ref{fig:choices}.
A \textit{deliberative stit} operator is also defined, which captures the notion that an agent can only truly be said to do something if it also has the choice of not doing it. See Fig.~\ref{fig:choices}.

 \begin{definition}[Deliberative stit]~\cite[Def. 2.8]{Horty01DLAgency}
	\label{def:dstit}
		With agent $\alpha$ and DAU statement $A$, 
		\[\Model,m/h \models \dstit{\alpha}{A} \text{ iff } \Choiceam(h) \subseteq |A|_m^\Model \text{ and } |A|_m^\Model \neq H_m\]
\end{definition}
Operators $cstit$ and $dstit$ are not interchangeable and fulfill complementary roles.

\paragraph{Optimal actions.}
To speak of an agent's obligations, we will need to speak of `optimal actions', those actions that bring about an ideal state of affairs.
Let $\Value: H_{0}\rightarrow \Re$ be a \textit{value function} that maps histories of $\Model$ to \textit{utility values} from the real line $\Re$.
This value represents the utility associated by all the agents to this common history.
Given two sets of histories $Z$ and $Y$ , we order them as 
\begin{equation}
\label{eq:proposition dominance}
Z\leq Y \text{ iff } \Value(h) \leq \Value(h') \quad \forall~h\in Z, h' \in Y
\end{equation}
Let 
$\State_{\alpha}^m \defeq \Choice{\Agent \setminus \{\alpha\}}{m}$
be the set of \textit{background states} against which $\alpha$'s decisions are to be evaluated.
These are the choices of action available to other agents.
Given two actions $K,K'$ in $\Choiceam$, 
$K\preceq K' \text{ iff } K \cap S \leq K'\cap S \textrm{ for all } S \in \State_{\alpha}^m$.
That is, $K'$ dominates $K$ iff it is preferable to it regardless of what the other agents do (known as \textit{sure-thing reasoning}).
Strict inequalities are naturally defined.
Optimal actions are given by
\begin{equation}
\label{eq:optimalam}
\Optimalam \defeq \{K \in \Choiceam \such \not\exists K' \in \Choiceam\textrm{ s.t. }K\prec K'\}
\end{equation}
$\Optimalam$ is non-empty in models with finite $\Choiceam$~\cite[Thm. 4.10]{Horty01DLAgency}.

\paragraph{Dominance Ought}
We are now ready to define Ought statements, i.e., obligations.
Intuitively we will want to say that at moment $m$, agent $\alpha$ \textit{ought to see to it} that $A$ iff $A$ is a necessary condition of all the histories considered ideal at moment $m$.
This is formalized in the following \textit{dominance Ought operator}, which is pronounced ``$\alpha$ ought to see to it that $A$ holds''.
\begin{definition}[Dominance Ought]
	\label{def:ought semantics}
With $\alpha$ an agent and $A$ an obligation in a model $\Model$, 
\begin{equation}
\label{eq:ought}
\Model,m/h \models \Ostit{\alpha}{A} \text{ iff } K \subseteq |A|_m^{\Model}~~\text{ for all } K \in \Optimalam
\end{equation}
\end{definition}
See Fig.~\ref{fig:choices} for examples.
The dominance ought satisfies a number of intuitive logical properties; we refer the reader to~\cite[Ch. 4]{Horty01DLAgency}.
The dual of the Ought is (weak, a.k.a. negative) Permission:
\[\Perm{\alpha}{A} \defeq \neg \Ostit{\alpha}{\neg A}\]
The intuitive meaning of permission is that $\alpha$ can ensure $A$ without violating any obligations.
Moreover, having a permission does not imply that one actually sees to it that $A$ is true.
This is quite different from $\eventually A$, which simply says that $A$ actually happens, and from $\exists \cstitaa$ which says that $\alpha$ can ensure $A$, neither of which refers to obligations.
\paragraph{Conditional obligation}
It is often necessary to say that an obligation is imposed only under certain conditions. 
Let $X$ be a proposition, i.e. $X=|\formula|_m$ for some $\formula$.
The choice of actions available to $\alpha$ at $m$ under the condition that $X$ holds is defined as $\Choiceam/X \defeq \{K\in \Choiceam \such K\cap X \neq \emptyset\}$. 
This is the right definition because non-determinism might make it impossible to have $K\subseteq X$ (i.e., an action that guarantees $X$), but future actions might still ensure the finally realized history will satisfy $X$.
Thus in Fig.~\ref{fig:choices} $\Choiceam/B=\{K_1\}$. 
Conditional dominance is then defined by comparing only histories that satisfy $\formula$:
for two actions $K,K'$ from $\Choiceam$, $K\preceq_X K'$ iff 
$K\cap S \cap X \leq K'\cap S \cap X$ for all $S \in \State_\alpha^m$.
%
The \textit{conditionally optimal actions} are then 
\begin{equation}
\label{eq:Optimalam/X}
\Optimalam/X \defeq \{K \in \Choiceam/X \such \not\exists K' \in \Choiceam/X \textrm{ s.t. } K\prec_X K'\}
\end{equation}
Finally, where $A$ is an obligation and $\formula$ a formula in the underlying temporal logic, the conditional Ought is defined by
\begin{equation}
\label{eq:conditional ought}
\Model, m/h \models \OstitC{\alpha}{A}{\formula} \text{ iff } K\subseteq |A|_m^\Model \,\, \forall K \in \Optimalam/|\formula|_m^\Model.
\end{equation}

We note that conditional obligation is \textit{not} the same as $\formula \implies \Ostit{\alpha}{\formula}$. 
Conditional obligation only compares $\formula$-satisfying histories, while this latter formula still compares all histories.


\paragraph{Terminology abuse}
In what follows, histories that belong to optimal actions will be called optimal. 

\section{Case Study in Modeling: Responsibility-Sensitive Safety for Self-Driving Cars}
\label{sec:formalizing rss in dau}
Responsibility-Sensitive Safety, or RSS, is a proposal put forth by Intel's Mobileye division~\cite{RSSv6}.
It proposes rules or requirements that, if followed by all cars in traffic, would lead to zero accidents.
RSS attempts to promote a natural way of driving by drawing the line between acceptable assertive driving, and unacceptable aggressive driving.
We consider these notions, of assertive vs. aggressive driving, to be fundamentally \textit{social} because they refer to what a particular society accepts. 
Thus we may say that RSS places an \textit{obligation} to avoid aggressive driving while giving \textit{permission} to drive assertively.
The RSS proposal is expressed in the language of continuous-time dynamical systems and ordinary differential equations, but the rules to be followed are not formalized logically, so it is not possible to reason about them or derive their logical consequences.
This work complements the dynamical equations-based presentation of RSS in~\cite{RSSv6} with a deontic logic formalism.
We have three objectives in doing so:
demonstrating the usefulness of DAU in a real use case; 
highlighting the ambiguities implicit in such proposals, which would go unnoticed without formalization;
and automating the checking of logical consistency and deriving of conclusions.
We first present the RSS rules in natural language (Section~\ref{sec:rss rules}), then their formalization (Section~\ref{sec:rss formalization}), and finally we analyze the rules' logical consequences. 
Three important points must be made:
\begin{enumerate}[(A)]
	\item The formalization does not depend on the dynamical equations that govern the cars because we wish our conclusions to be independent of these lower-level concerns. 
	This is consistent with the standard AV control architecture where a logical planner decides what to do next (`change lanes' or `turn right') and a lower-level motion planner executes these decisions. 
	Our logical analysis concerns the logical planner.
	\label{note first}
	\item We are not trying to formalize general traffic laws~\cite{Althoff15TrafficRules} or driving scenarios, which is outside the scope of this paper. We are only formalizing the RSS rules. 
	\item Every formalization, in any logic, can always be refined. We are not aiming for the most detailed formalization; we aim for a useful formalization.
	\label{note last}
\end{enumerate}

\subsection{The RSS rules}
\label{sec:rss rules}
The rules for Responsibility-Sensitive Safety are~\cite{RSSv6}:
\begin{enumerate}[RSS1.]
	\item Do not hit someone from behind. \label{rss-behind}
	\item Do not cut-in (to a neighboring lane) recklessly. \label{rss-cutin}
	\item Right-of-way is given, not taken. \label{rss-row}
	\item Be careful of areas with limited visibility. \label{rss-vis}
	\item If you can avoid an accident without causing another one, you must do it. \label{rss-avoid}
	\item To change lanes, you do not have wait forever for a perfect gap: i.e., you do not have to wait for a gap large enough to get into even when the other car, already in the lane, maintains its current motion.\label{rss-assertive}
\end{enumerate}

RSS\ref{rss-assertive} is derived directly from the following in~\cite[Section 3]{RSSv6}: ``the interpretation [of the duty-of-care law] should lead to [...] an agile driving policy rather than an overly-defensive driving which inevitably would confuse other human drivers and will block traffic [...].
As an example of a valid, \underline{but not useful}, interpretation is to assume that in order to be ``careful'' our
actions should not affect other road users. Meaning, if we want to change lane we should find a gap large enough such
that if other road users continue their own motion uninterrupted we could still squeeze-in without a collision. Clearly,
for most societies this interpretation is \underline{over-cautious and will lead the AV to block traffic and be non-useful}.''
Note that, consistently with points (\ref{note first})-(\ref{note last}) above, this is stated without any reference to dynamics or specific scenarios. 
The RSS authors are concerned that overly cautious driving might lead to unnatural traffic, so RSS aims to allow cars to move a bit assertively, and defines correct reactions to that.

Note finally that RSS\ref{rss-vis} is explicitly formulated in terms of obligations and ability. However, we will not study RSS\ref{rss-vis} and RSS\ref{rss-avoid} as they are currently too vague for formalization.

\subsection{Formalization of RSS Rules}
\label{sec:rss formalization}
\noindent\textbf{Formalizing RSS\ref{rss-behind}}. 
Let $\formula$ denote `collision with car ahead of me'.
A plausible formalization of RSS\ref{rss-behind} is then
\[RSS\ref{rss-behind}.\,\, \Ostit{\alpha}{\neg \formula}\]
That is, $\alpha$ ought to see to it that there is no collision with a car ahead of it.
A positive aspect of this formalization is that if at some $m$, a rear-end collision is inevitable, then $RSS\ref{rss-behind}$ ceases to hold: $\forall \formula \implies \neg \Ostit{\alpha}{\neg \formula}$.
This provides an automatic and interpretable update of control objectives.
In a deployed system, an automatic proof engine could update which obligations hold and which don't, based on the current situation~\cite{Arkoudas05towardethical}.
An alternative formalization is 
\[RSS\ref{rss-behind}r.\,\, \Ostit{\alpha}{\neg [\alpha \, dstit\!:\formula]}\]
This says that $\alpha$ should see to it that it does not deliberately ensure an accident $\formula$.
This form of obligation is called \textit{refraining}~\cite{HortyB95Dstit}: $\alpha$ \textit{refrains} from hitting anyone from behind.
If a rear-end collision is inevitable at some $m$, then $RSS\ref{rss-behind}r$ still holds (unlike $RSS\ref{rss-behind}$) \emph{and is trivially satisfied}.
This might be computationally cheaper than having to use a proof engine to tell us that the obligation no longer holds.

In the general case, some actions at $m$ guarantee a collision, some guarantee no collision, and the rest don't guarantee either: the future could evolve either way.
If we are interested in guaranteeing no collision over a long horizon, then, because of non-determinism, it is unlikely that any action in the present moment can guarantee that.
In such a case $RSS\ref{rss-behind}$ will be violated repeatedly in a rather trivial way; on the other hand, $RSS\ref{rss-behind}r$ is more permissive, since it can be met by taking any optimal action that allows the \textit{possibility} of no collision over the horizon. 
A lower-level controller, running at a higher rate, could then ensure freedom from collision forever.
\\

\noindent\textbf{Formalizing RSS\ref{rss-cutin}}. 
Define formulas, $\psi: $ a non-reckless cut-in, and $\psi_r$: a reckless cut-in.
Then RSS\ref{rss-cutin} is formalizable as 

\[RSS\ref{rss-cutin}. \Ostit{\alpha}{\always (\psi\lor \psi_r \implies \neg \psi_r)}.\]

That is, $\alpha$ should see to it that always, if a cut-in happens, then it is a non-reckless cut-in.
\\

\noindent\textbf{Formalizing RSS\ref{rss-row}}.
Formalizing this rule requires some care. 
First, note that RSS\ref{rss-row} should probably be amended to say that `Right-of-way is given, not taken, \emph{and some car is given the right-of-way}' - otherwise, traffic comes to a standstill.
We will first focus on formalizing the prohibition (nobody should take the right-of-way), then we will formalize the positive obligation (somebody must be given it).

Let $\Agent = \{\alpha,\beta,\gamma,\ldots\}$ be a finite set of agents.
Define the atomic propositions
$GROW_\beta^\alpha$: $\beta$ gives right-of-way to $\alpha$ and
$p_\alpha$: $\alpha$ proceeds/drives through the conflict region.
Then
$TROW_\alpha \defeq p_\alpha \land \neg (GROW_\beta^\alpha \land GROW_\gamma^\alpha \land \ldots)$
formalizes taking the right-of-way: $\alpha$ proceeds without being given the right-of-way by everybody.
 We could now express the prohibition in RSS\ref{rss-row}: every $\alpha$ ought to see to it that it does not take the right-of-way:
 \begin{equation}
 \label{eq:rss-row prohib}
 RSS\ref{rss-row}prohib0. \, \bigwedge_{\alpha \in \Agent} \Ostit{\alpha}{\neg TROW_\alpha}
 \end{equation}

The difficulty with this formulation is that it could lead to $\alpha$ being obliged to force \textit{everybody else} to give it the right-of-way - something over which, a priori, it has no control. 
To see this, we need the following.
\begin{proposition}
	\label{prop:force others}
	Given obligations $A$ and $B$, 	$\Ostit{\alpha}{A \lor B} \land (\forall \neg A) \implies \Ostit{\alpha}{B}$
\end{proposition}
In words, if $\alpha$ ought to ensure $A$ or $B$ at $m/h$, but every available history violates $A$, then its obligation is effectively to ensure $B$.
\begin{proof}
Assume that $m/h \models \Ostit{\alpha}{A \lor B} \land \forall \neg A$.
By definition of the dominance ought, 
for all $K \in \Optimalam, K\subseteq |A\lor B|_m$.
And by definition of $|A|_m$ (Eq.~\ref{eq:Amm}), $|A\lor B|_m = |A|_m \cup |B|_m$.
We also have that $m/h\models \forall \neg A$, i.e., $m/h' \models \neg A$ for all $h' \in H_m$; thus $|A|_m = \emptyset$.
Therefore $\forall K \in \Optimalam, K\subseteq |B|_m$, which is the definition of $m/h \models \Ostit{\alpha}{B}$.
\end{proof}
Applied to Eq.~\eqref{eq:rss-row prohib} with $A=\neg p_\alpha$ and $B=\land_{\beta\neq \alpha}GROW_\beta^\alpha$, Thm.~\ref{prop:force others} says that if $\alpha$ is in a situation where it has no choice but to proceed (e.g. as a result of slippage on a wet road, say), then its obligation is to see to it that everybody else gives it the right-of-way, which is unreasonable.

Instead, we adopt a more passive attitude: every agent sees to it that if they are not given the right-of-way, then they do not pass.
Letting atomic proposition $g_\alpha$ denote that right-of-way is Granted to $\alpha$,
\begin{equation}
\label{eq:rss-row prohib 2}
RSS\ref{rss-row}prohib. \, \bigwedge_{\alpha \in \Agent} \Ostit{\alpha}{\always (\neg g_\alpha \implies \neg p_\alpha)}
\end{equation}

The positive obligation, that somebody must be given the right-of-way, seems to be a \textit{group obligation}: the \textit{group} must give right-of-way to one of its members.
Group obligations are formally defined in \cite[Ch. 6]{Horty01DLAgency}. 
Then we formalize
 \begin{equation}
 \label{eq:rss-row pos}
 RSS\ref{rss-row}pos. \, \Ostit{\Agent~}{\lor_{\alpha \in \Agent} g_\alpha}
 \end{equation}
This says the group $\Agent$ has an obligation to give right-of-way to someone, and the only choice is in \textit{who} gets it.
Finally, we formalize $RSS\ref{rss-row}$ as the conjunction of  $RSS\ref{rss-row}prohib$ and $RSS\ref{rss-row}pos$.
\\

\noindent\textbf{Formalizing assertiveness and RSS\ref{rss-assertive}}.
This rule says that if the car wants to change lanes, it shouldn't have to wait forever for the perfect gap (otherwise, traffic is stalled).
It is one way in which RSS attempts to promote `assertive driving', a style of driving that tries to obtain right-of-way in a `polite' way.
The key difficulty, of course, is to distinguish between assertive driving, which is acceptable, and aggressive driving, which is not. 
Deontic logic can help in that regard.
\emph{We model assertiveness as a permission to not drive conservatively or defensively.} 
That is, if $\chi$ is a formula that describes conservative driving behavior in a particular context $\Omega$, then driving assertively is the conditional permission
\begin{equation}
\label{eq:def assertive}
\PermC{\alpha}{\neg \dstit{\alpha}{\chi}}{\Omega}
\end{equation}
This is a \textit{permission}: it does not constitute an obligation to drive assertively.
Depending on its reward structure, the agent might choose to drive conservatively after all.
Importantly, Eq. \eqref{eq:def assertive} states that the agent can drive assertively without violating any obligations it does have.

For RSS\ref{rss-assertive}, conservative driving consists in waiting for the perfect gap before passing, that is, waiting until the other car, already in the lane, gives $\alpha$ the right-of-way. 
Thus we may take $\chi = g_\alpha \release \neg p_\alpha$, 
where, recall, $p_\alpha$ means `$\alpha$ proceeds through the conflict region' and $g_\alpha$ means `$\alpha$ is granted the right-of-way'.
Finally, with $w_\alpha$ meaning `$\alpha$ wants to change lanes' we have
\begin{equation}
\label{eq:rss-assertive}
RSS\!\ref{rss-assertive}.\quad \PermC{\alpha}{\neg \dstit{\alpha}{g_\alpha \release \neg p_\alpha}}{w_\alpha}
\end{equation}

\subsection{Application: undesirable consequence of RSS star-calculations}
\label{sec:other accidents}
One of the main tenets of RSS is that an Autonomous Vehicle (AV) is only responsible for avoiding potential accidents between itself and other cars (so-called `star calculations'); interactions between 2 other cars are not its concern~\cite[Remarks 1 and 8]{RSSv6}.
Yet everyday driving experience makes clear that our actions can be faulted for at least \textit{facilitating} an accident: e.g., by repeated braking, I may cause the car behind me to do the same, leading the car behind \textit{it} to rear-end it.
Or I might make a sudden lane change over two lanes, causing the car in the lane next to me to over-react when I speed past it, and collide with someone else.
We now show how this intuition is automatically captured by the DAU logic, and that RSS star-calculations lead to undesirable behavior of the AV.

Let $\formula \in $ \CTLs~denote a formula expressing ``Accident between two other cars'', and the accident is such that $\alpha$ can facilitate it as in the above 2 examples.
Then $\dstit{\alpha}{\formula}$ says that $\alpha$ (deliberately) sees to it that the accident happens even though it could avoid doing so;
given what we assumed about this accident, this means $\alpha$ facilitates the accident.
Then $\dstit{\alpha}{\neg \dstit{\alpha}{\formula}}$ expresses that $\alpha$ sees to it that it does \textit{not} facilitate the accident: this is a form of refraining.
Finally,
$\dstit{\alpha}{\neg \dstit{\alpha}{\neg \dstit{\alpha}{\formula}}}$ says that $\alpha$ refrains from refraining, that is, $\alpha$ does not refrain from facilitating the accident (even though it could). 
The RSS position is that it is OK for $\alpha$ to refrain from refraining~\cite[Remarks 1 and 8]{RSSv6}.
However, refraining from refraining is the same as doing. Formally~\cite[2.3.3.]{Horty01DLAgency}
\[ [\alpha \,dstit: \neg [\alpha\,dstit: \neg [\alpha\,dstit:\formula]]] \equiv \dstit{\alpha}{\formula}\]
This matches our intuition: to not refrain from facilitating an accident even though one could (left-hand side in previous equation) is the same as facilitating it (right-hand side).
In other words, under this formalization, the RSS position is tantamount to allowing AVs to facilitate accidents between others - clearly, an undesirable conclusion.
This aspect of RSS, therefore, needs refinement to take into account longer-range interactions between traffic participants.

\section{Obligation Propagation}
\label{sec:obligation propagation}
Obligations vary over time: the obligation at moment $m$ is the set of necessary conditions (formulas in the tense logic) satisfied by all histories optimal at $m$, and the set of optimal histories can change from moment to moment. 
There is thus a need to understand how obligations change over time:
for example if the agent does not act optimally at $m$, does the obligation disappear at the next moment? 
Or does it persist, perhaps in a modified form?
The formal study of obligation propagation is also a way to interpret the temporal evolution of utility-maximizing controllers: as the controller (and the environment) act, obligations change, placing new constraints on the controller.

The following examples show that these questions must be studied formally, since intuition usually fails us.
Consider the following tentative propagation pattern, in which $\formula$ is a \TenseLogic~formula:
	\begin{equation}
	\label{eq:x false prop ->}
	\Ostit{\alpha}{\Next \formula} \implies \Next \Ostit{\alpha}{\formula} 	
	\end{equation}
This says that an obligation now to ensure that $\formula$ holds at the next moment implies an obligation at the next moment to ensure that $\formula$ holds, which sounds plausible. 
However, it is not valid in DAU.
Fig.~\ref{fig:cex-next} gives a counter-example: $K_2$ is optimal at $m_1$ so $m_1/h_1 \models \Ostit{\alpha}{\Next \phi}$;
however $m_3$ is the next moment along $h_1$ and $m_3/h_1  \not \models \Ostit{\alpha}{\phi}$, so $m_1/h_1  \not \models \Next \Ostit{\alpha}{ \phi}$ and Eq. \eqref{eq:x false prop ->} is not valid. 

As a second example, the following tentative pattern says that if the agent has an obligation to ensure that $ \formula$ eventually holds, does not do so now, \textit{but it is still possible to do so at the next moment}, then at the next moment the agent still has an obligation to ensure eventually $ \formula$:
	\begin{equation}
	\label{eq:big false prop}	
	\Ostit{\alpha}{\eventually\formula} \land \neg \cstit{\alpha}{\eventually \formula} \land \Next \exists \cstit{\alpha}{\eventually\formula} \implies  \Next \Ostit{\alpha}{\eventually \formula}	
	\end{equation}
Fig. \ref{fig:cex-big} shows a counter-example to this second pattern: we have $m_1/h_2 \models \Ostit{\alpha}{\eventually \phi}$, and that $m_1/h_2 \models \Next \exists \cstit{\alpha}{\eventually \phi}$. If $K_2$ will be taken, then $m_2/h_2 \not\models \Ostit{\alpha}{\eventually \phi}$ because $K_3$ is optimal at $m_2$, so Eq. \eqref{eq:big false prop} is also invalid.
	\begin{figure}[t]
		\centering
		\begin{subfigure}{0.32\textwidth}
            \includegraphics[width=\linewidth]{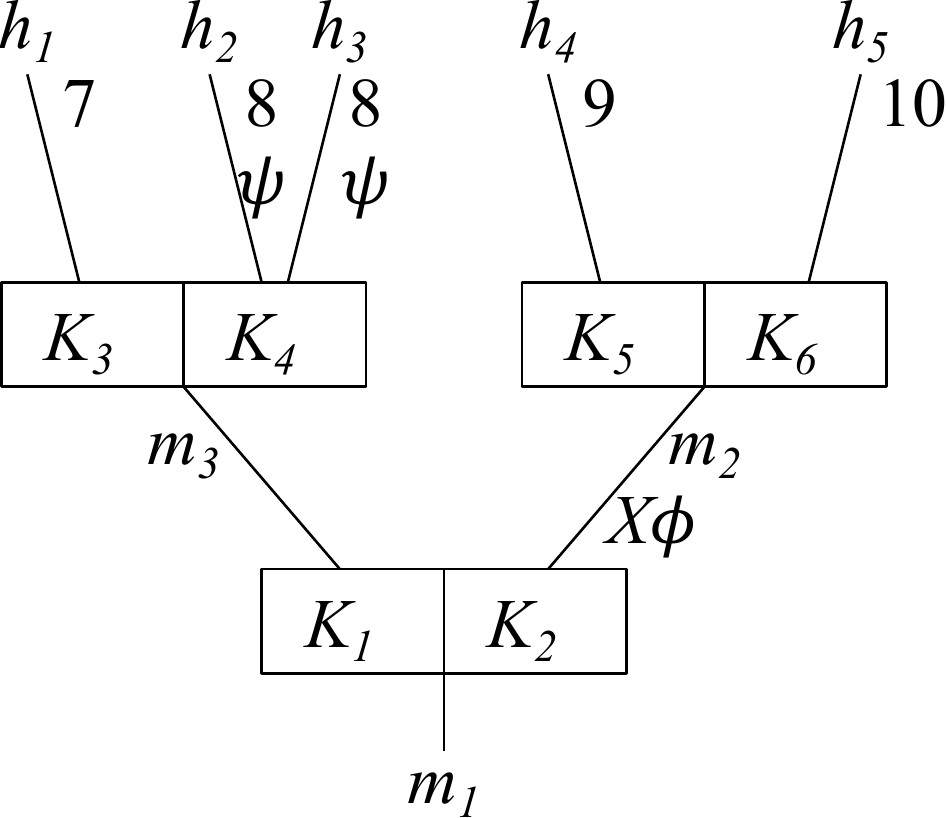}
            \caption{} \label{fig:cex-next}
          \end{subfigure}%
          \hspace*{0.15\textwidth}   
          \begin{subfigure}{0.32\textwidth}
            \includegraphics[width=\linewidth]{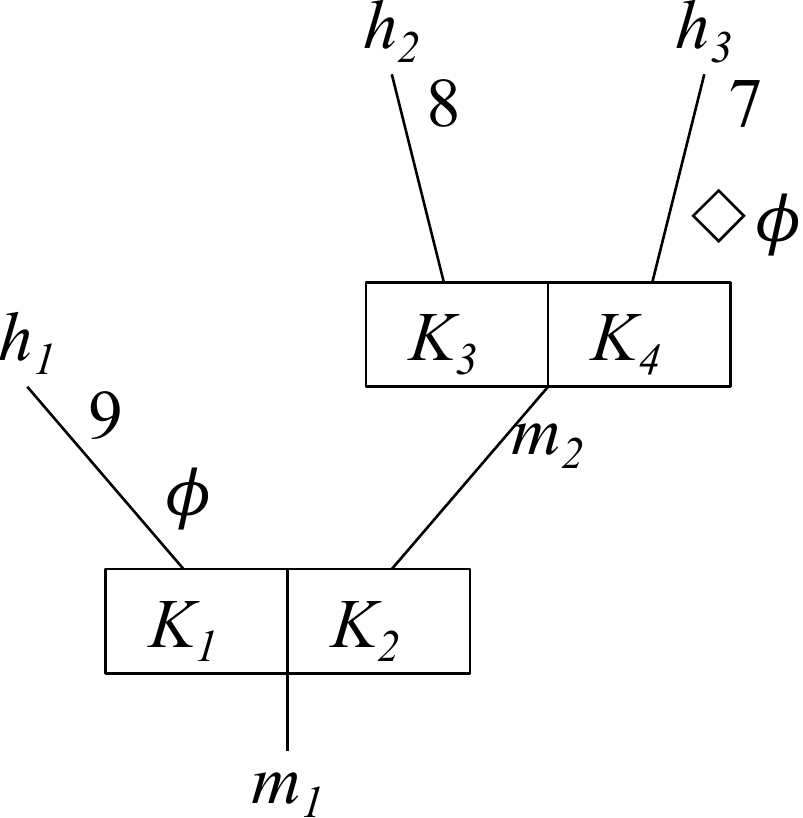}
            \caption{} \label{fig:cex-big}
          \end{subfigure}%
		\caption{Counter-examples to tentative obligation propagation patterns. (a) Pattern in Eq.~\eqref{eq:x false prop ->};  (b) Pattern in Eq.~\eqref{eq:big false prop}}
		\label{fig:cex}
	\end{figure}
	
Both counter-examples exploited the fact that along the $m/h$ pair where the left-hand side is evaluated, the agent acts non-optimally.
This suggests that to derive valid temporal propagation patterns, we must assume the agent is acting optimally. 
So we define the distinguished atomic proposition $\atakopt$ for this purpose:
\[m/h\models \atakopt \text{  iff  } \Choiceam(h) \in \Optimalam\]

The following pattern \textit{is} valid in DAU.
Let $m^+(h)$ be the moment that follows $m$ in $h$; e.g., $m_1^+(h_1) =m_3$ in Fig.~\ref{fig:cex-next}. 
\begin{equation}
\label{eq:pattern1}
\Ostit{\alpha}{\Next \formula} \land \atakopt \implies \Next \Ostit{\alpha}{\formula}
\end{equation}
\begin{proof}
	The pair $m/h$ satisfies the left-hand side iff $K \subseteq |\Next \formula|_m$ for all optimal $K$ at $m$. By $\atakopt$, we have that $\Choiceam(h)$ is optimal, thus $m/h \models \Next \formula$, which implies that $m^+(h)/h \models \formula$, which is the definition of $m/h \models  \Next \Ostit{\alpha}{\formula}$.
\end{proof}
	
Acting optimally is not always enough however.
The following valid pattern says that if $\alpha$ ought to see to it that $\eventually \formula$, acts optimally, but it is impossible to satisfy $\formula$ now, then at the next moment $\alpha$ still ought to see to it that  $\eventually \formula$.
Here, $\forall \neg \formula$ is necessary in the antecedent: the implication fails trivially without it.
\begin{equation}
\Ostit{\alpha}{\eventually \formula} \land  \atakopt \land \forall \neg \formula \implies \Next \Ostit{\alpha}{\eventually\formula} \label{eq:simple prop ev}
\end{equation}

Finally we present a pattern of obligation propagation which does not require optimal behavior, but which is only satisfied in certain models.

\begin{lemma}
	\label{lemma:ob prop general}
	With $\formula,\psi$ \TenseLogic~formulas, let $\Model$ be a stit model which satisfies the following constraint at every moment $m$:
	for all actions $K \in \Choiceam$ s.t. $K\subseteq |\neg \formula|_m$ and which contain a history $h$ s.t. 
	$m^+(h) \models \exists \cstit{\alpha}{\psi}$,
	it holds that all optimal actions in $\Optimal{\alpha}{m^+(h)}$ guarantee $\psi$.
	Then \emph{in such a model}, the following is satisfied at every index $m/h$.
	\begin{equation}
	\label{eq:ob prop general}
	\Ostit{\alpha}{\formula \lor \Next \psi} \land \cstit{\alpha}{\neg \formula } \land \Next \exists \cstit{\alpha}{\psi} \implies \Next \Ostit{\alpha}{\psi}
	\end{equation}	
\end{lemma}
This says that if $\alpha$ has an obligation to ensure $\formula \lor \Next \psi$, guarantees $\neg \formula$ now, but next it is still possible to guarantee $\psi$, then the next obligation is to guarantee $\psi$.
\begin{proof}
	Let $m/h$ be an index in $\Model$ at which the DAU formula \eqref{eq:ob prop general} is evaluated.
	Let $K_h$ be the action to which $h$ belongs,
	and for brevity, write $m' = m^+(h)$.
	
	\underline{Case 1: $h \in \cup_{K\in \Optimalam} K$.}
	  Then $K_h\subseteq |\formula \lor \Next \psi|_m = |\formula|_m \cup |\Next \psi|_m$.
	By hypothesis, $K_h \subseteq |\neg \formula|_m$ also so 
	$K_h \subseteq |\Next \psi |_m \setminus |\neg \formula|_m$.	
	By construction, $\Optimal{\alpha}{m'}\subseteq H_{m'} \subseteq K_h$ so for every $K^* \in \Optimal{\alpha}{m'}$ and every $h'\in K^*$, $m'/h'\models \psi$, 
	which is the definition of $m/h \models \Next \Ostit{\alpha}{\psi}$.
	
	\underline{Case 2: $h \notin \cup_{K\in \Optimalam} K$.}
	From the formula antecedent, we have that $K_h \subseteq |\neg \formula|_m$ and that $m'/h \models \exists \cstit{\alpha}{\psi}$. 
	Therefore the model constraint yields that all optimal actions at $m'$ guarantee $\psi$, which is the definition of $m/h\models \Next \Ostit{\alpha}{\psi}$.	
\end{proof}

Finally, the proof also establishes the following pattern.
\begin{proposition}
	\label{prop:ob prop}
	The following is valid (i.e., satisfied in all models) in DAU:
	\[\Ostit{\alpha}{\formula \lor \Next \psi} \land \cstit{\alpha}{\neg \formula } \land \atakopt \implies \Next \Ostit{\alpha}{\psi}\]
\end{proposition}

\section{Model Checking DAU}
\label{sec:mc dau}
\label{sec:model checking}
The expressive power of DAU makes the logic a useful tool in the hands of a system designer. The system designer can use DAU to specify the obligations the system ought to have.
While DAU derives obligations from stit trees, control engineers often model agents as some kind of automata. How then can we verify that the controller has the obligations the system designer has specified?
Note that having an obligation is not the same as meeting that obligation: the obligation is a constraint that might or might not be met. This section's algorithms verify that a system \textit{has} \yhl{a given} obligation, i.e. that it has the \yhl{given} constraints on its behavior.

We can ensure that an agent has an obligation by framing the question as a model checking problem. In this section we cast agents as \textit{stit automata}, and introduce novel algorithms to perform model checking for obligations. All proofs not given here can be found in the 
supplementary material.

\subsection{Stit Automata}
\label{sec:agent automata}

For a set $S$, let $S^\omega$ denote the set of infinite sequences $(a_i)_{i\in \Ne}$ with $a_i\in S$.
\begin{figure*}[t]
	\centering
	\includegraphics[height=2.25in,keepaspectratio]{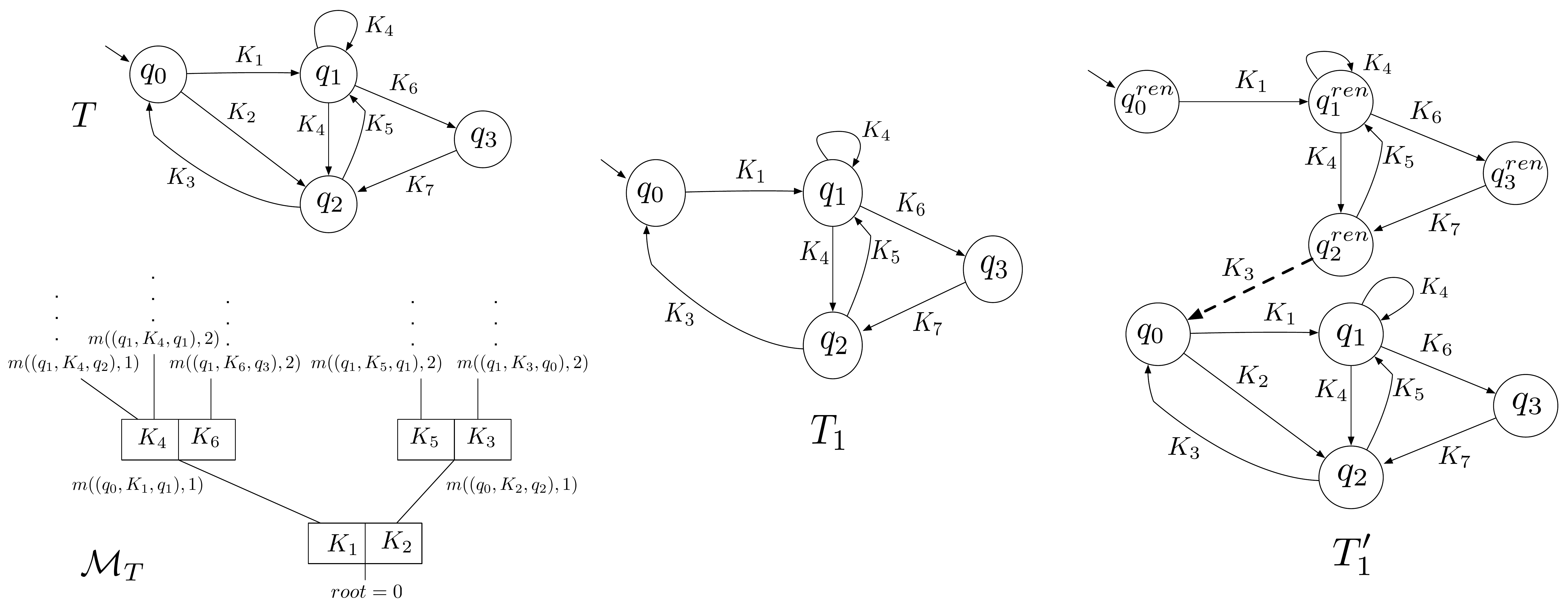}
	\caption{Left: a stit model generated by executing the stit automaton $T$ (transition weights not shown). Center and right: Automata $T_n$ and $T_n'$ used in Algorithm~\ref{algo:mc}. $T_1$ only has $K_1$ as first action, and $T_1'$ is obtained by re-naming states of $T_1$ and adding a copy of $T_1$ to it. Executions of $T_1'$ are simply the execution of $T$ that start with $K_1$.}
	\label{fig:algo}
\end{figure*}

\begin{definition}[Stit automaton]
	\label{def:stit automaton}
	Let $AP$ be a finite set of atomic propositions.
	A \emph{stit automaton} $T$ is a tuple $T = (Q, q_I, \Act, F, \Transrel, L, w,\lambda)$, where
	$Q$ is a finite non-empty set of states,
	$q_I$ is the initial state,
	$\Act$ is a finite non-empty set of actions, 
	$F \subset Q$ is a set of final states,
	$\Transrel \subset Q\times \Act \times Q$ is a finite transition relation such that if $(q,K,q')$ and $(q,K',q')$ are in $\Transrel$ then $K=K'$,
	$L:Q\rightarrow 2^{AP}$ is a labeling function,
	$w: \Transrel \rightarrow \Re$ is a weight function,
	and $\lambda: \Re^\omega \rightarrow \Re$ is an accumulation function.
\end{definition}
When dealing with multiple automata, we will sometimes write $T.q_I, T.\lambda$, etc, to clarify which automaton is involved.
Note that $T$ is a type of non-deterministic weighted automaton. 
Its unweighted counterpart $T^u$ is a classical transition system, thus for a \TenseLogic~formula $\formula$, we could model-check whether $T^u\models \formula$.
Denote by $\Transrel(q)$ the set of outgoing transitions from $q$ ($\Transrel(q) = \{(q,K,q') \in \Transrel\}$), 
by $\Act(q) = \{K \in \Act \such \exists (q,K,q') \in \Transrel\}$ the set of actions available at $q$.
Examples of $\lambda$ include the functions $\min/\max$, discounted sum and long-run average:
\begin{eqnarray}
\label{eq:discounted sum}
\min(\xi) &=& \min_i w(\xi[i])
\\
\discsum(\xi) &=& \sum_{i=0}^{\infty} \gamma^i \cdot w(\xi[i])~,0<\gamma<1
\\
\liminfavg(\xi) &=& \liminf_{n\rightarrow \infty} \frac{1}{n}\sum_{i=0}^{n}  w(\xi[i])
\end{eqnarray}

\begin{definition}[Execution]
	\label{def:stit aut execution}
	Let $T$ be a stit automaton and $q_0$ a state in $Q$.
	A \textit{$q_0$-rooted execution} $\xi$ of $T$ is a sequence of transitions of the form $\xi = (q_0,K_0,q_1)(q_1,K_1,q_2)\ldots \in \Transrel^\omega$.
	The corresponding sequence of actions $K_0,K_1,\ldots \in \Act^\omega$ is called a \textit{tactic}.
	The value of execution $\xi=\xi[0]\xi[1]\xi[2]\ldots $, where $\xi[i] \in \Delta$, is defined to be $\lambda(w(\xi[0])w(\xi[1])w(\xi[2])\ldots)$, and abbreviated $\lambda(\xi)$.	
\end{definition}
Because of non-determinism, a tactic can produce multiple executions.
A set of agents is modeled by the product of all individual stit automata, which is itself a stit automaton.
(When taking the product, we must define how weights are combined and how to construct the product's accumulation function, which are application-specific considerations.)  
Therefore the rest of this section applies to stit automata, whether they model one or multiple agents. 
We will continue to refer to one agent $\alpha$ for simplicity.

\paragraph{From stit automata to stit models.}
An automaton $T$, along with a state $q_0\in Q$, induce a stit model $\Model_{T,q_0}$ in the natural way, which we now describe somewhat informally:
state $q_0$ maps to the root moment 0 of $\Model_{T,q_0}$.
From $q_0$, $T$ has a choice of actions $\Act(q_0)$, which map to the actions available at 0 in $\Model_{T,q_0}$.
Each action $K$ in $\Act(q_0)$ non-deterministically causes one or more transitions, each of which maps to a moment in $\Model_{T,q_0}$; 
all transitions caused by a given $K$ map to moments in histories that originate in the same action $K$ in $\Model_{T,q_0}$.
And so on from each next state.
See Fig.~\ref{fig:algo} for an example.
We let $\eta: \Delta \rightarrow \Tree$ denote the map from transitions to moments, and lift it to executions in the natural way, i.e., $\eta(\xi)\defeq \eta(\xi[0])\eta(\xi[1])\ldots$
By construction, $\eta(\xi)$ is a history in $\Model_{T,q_0}$.
Its utility $\Value(\eta(\xi))$ is the $\lambda$-value of the generating execution, i.e., $\lambda(\xi)$.
The atoms labeling $\eta((q,K,q'))$ are $v(\eta((q,K,q'))) \defeq L(q')$.
The formal construction and proof of the following proposition 
can be found in the supplementary material.
\begin{proposition}
	\label{prop:Mt}
	The structure $\Model_{T,q_0}$ is a utilitarian stit model with finite $\Choiceam$ for every agent $\alpha$ and moment $m$.
\end{proposition}


Model-checking determines whether a stit automaton, at a given state, satisfies an Ought statement.
\begin{definition}
	\label{def:sat for automaton}
	Given an automaton $T$, one of its states $q$, the induced model $\Model_{T,q}$ and an obligation $A$, 
	we say that $T,q$ satisfies $\Ostit{\alpha}{A}$, written $T,q\models \Ostit{\alpha}{A}$, iff $\Model_{T,q}, 0/h\models \Ostit{\alpha}{A}$ for an arbitrary history $h \in H_0$.
\end{definition}
The history $h$ is arbitrary since the truth of $\Ostit{\alpha}{A}$ does not depend on the history but only on the moment.

The construction of $\Model_{T,q}$ roots all histories at moment 0.
However, what if the automaton can only reach state $q$ after $i$ time steps? 
Then a priori, it might be that whether an Ought holds at $q$ depends on $i$, because the accumulation function $\lambda$ can be time-dependent.
The following shows that for certain accumulation functions important in practice, the choice of root moment does not matter.
\begin{proposition}
	\label{prop:moment of evaluation}
	Given a stit automaton $T$ and an obligation $A$,
	let $q,q'$ be states of $T$ s.t. $q$ is reachable from $q'$ in $i$ transitions along an execution $\xi$.
	Let $h$ be an arbitrary history of $\Model_{T,q}$,
	let $h' = \eta(\xi)$ be the history that connects $\eta(\xi[0])$ to $\eta(\xi[i-1])$ in $\Model_{T,q'}$,
	and let $m' = \eta(\xi[i-1])$.
	Then, if $\lambda$ is discounted sum or long-run average,
	$\Model_{T,q}, 0/h\models \Ostit{\alpha}{A}$ iff $\Model_{T,q'}, m'/h'\models \Ostit{\alpha}{A}$ 
\end{proposition} 

\begin{proof}
For clarity, we write $\Model = \Model_{T,q}$ and $\Model'=\Model_{T,q'}$, and write $\Model . ()$ vs $\Model'.()$ to disambiguate something in $\Model$ vs something in $\Model'$.
We will show that the trees rooted at $0/h$ in $\Model$ and $i/h'$ in $\Model'$ have the same structure and that the value ordering of their histories is the same in both models.
This implies that the same Oughts hold at both.
 
The histories of $\Model$ are images, under $\eta$, of executions that start at $q$.
Because transition $\xi[i-1]$ ends in $q$, the histories of $\Model'$ rooted at $m' =\eta(\xi[i-1])$ are also images of executions that start at $q$. 
Therefore, $\Model'.H_{m'}$ is identical to $\Model.H_0$.
In particular they satisfy the same set of \TenseLogic~formulas.
We refer to this common set of histories as $H^*$.

Take two arbitrary $h_1,h_2\in H^*$ and their pre-images $\xi_1,\xi_2$ by $\eta$.
By construction, $\xi[k+i]=\xi_1[k]=\xi_2[k]$, $k\geq 0$, and the concatenation $f_j \defeq h'[0]\ldots h'[i-1]h_j[0]h_j[1]\ldots$ is a 0-rooted history in $\Model'$, $j=1,2$.
If $\lambda=\discsum$ then $\Model.\Value(h_j) = \sum_{k\geq 0}\gamma^k w(\xi_j[k])$, while $\Model'.\Value(f_j) =  \sum_{k= 0}^{k=i-1}\gamma^{k}w(\xi[k])+ \sum_{k\geq 0}\gamma^{i+k}w(\xi_j[k])$.
Thus 
\[\Model.\Value(h_1) \leq \Model.\Value(h_2) \text{     iff     }\Model'.\Value(f_1) \leq \Model'.\Value(f_2)\]
Thus the histories in $H^*$ are identically ranked in both models, which implies that optimal actions are the same.
This, combined with the fact that they satisfy the same formulas, yields the desired conclusion.

Similarly, if $\lambda$ is $\liminfavg$, then for $j=1,2$,
\[\Model'.\Value(f_j) =   \liminf_{n\rightarrow \infty}\frac{1}{n} \left[ \sum_{0\leq k\leq i-1}w(\xi[k]) + \sum_{k\geq 0}w(\xi_j[k])\right] =  \liminf_{n\rightarrow \infty}\frac{1}{n}  \sum_{k\geq 0}w(\xi_j[k]) = \Model.\Value(h_j)\]
So histories of $H^*$ are identically ranked by $\liminfavg$ in both models, yielding the desired conclusion.
	\end{proof}


\subsection{Model Checking of Unconditional Obligations}
\label{sec:mc unconditional}
The problem of \emph{cstit model checking} is: given a stit automaton $T$ that models an agent $\alpha$, a state $q \in T.Q$, and a formula $A$ which is either a \TenseLogic~formula, or a statement of the form $\dstit{\alpha}{\formula}$ or $~\neg \dstit{\alpha}{\formula}$ where $\formula$ is a \TenseLogic~formula, determine whether $\Model_{T,q},0/h \models \Ostit{\alpha}{A}$ for some arbitrary $h \in H_{0}$.

We restrict the algorithm to statements of the above forms for conciseness of the presentation; DAU formulas with additional nesting levels can be handled by extending the algorithms we present below.

Recalling Definition \ref{def:ought semantics}, the cstit model checking problem can be broken into two parts: what is the set of optimal actions at $H_{0}$ (i.e. $K \in \Optimalar$), and do all these optimal actions guarantee the truth of $A$ (i.e. $K \subseteq |A|_{0}^{\Model}$)? If all optimal actions guarantee $A$ then, by Def. \ref{def:ought semantics}, $\Model_T$ has obligation $A$ at $0/h$. Algorithm \ref{algo:mc} solves this problem, and is discussed in depth below.

\begin{proposition}
	\label{prop:mc is possible}
	Algorithm~\ref{algo:mc} returns True iff $\Model,root/h\models \Ostit{\alpha}{A}$. 
	It has complexity $O(2\sigma(|T|+c_\lambda+|T|\cdot 2^{|\formula|}))$, where $\sigma$ is the maximum out-degree from any state in $T$, $c_\lambda$ is the cost of computing the minimum and maximum values of a tactic executed on automaton $T$, $|T|$ is the number of states and transitions in $T$, and $|\phi|$ is the size of the \TenseLogic formula in $A$.
\end{proposition}

Algorithm \ref{algo:mc} begins by considering each action available to the agent at root: $K_n \in \Choicear$. For each of these actions, a version $T_n'$ of the automaton $T$ is constructed such that each of its executions is an execution of $T$ starting with action $K_n$. In this way we can determine the best ($\vnu$) and worst ($\vnl$) possible values of the executions in each action $K_n$ by analyzing the automaton $T_n'$ (this is discussed further in Section \ref{sec:computing un ln}). With the range of values $[\vnl, \vnu]$ known for each action $K_n$, we find those ranges whose $\vnu$ is not less than any $\vnl'$. These value ranges are un-dominated. The optimal actions $\Optimalar$ are those actions whose corresponding value ranges are un-dominated. This completes the first step of the algorithm: finding the optimal actions at $H_{0}$. The second step determines if all optimal actions guarantee $A$. In this algorithm $\models_{\text{\CTLs}}$ denotes the classical \CTLs~satisfaction relation. If the obligation is a \CTLs formula, then we simply check if every execution of $T_n'$ satisfies the $A$ by checking $\forall A$. If the obligation is a $dstit$ statement containing a \CTLs formula $\phi$, then we must verify two conditions: that not all actions in $\Choicear$ guarantee $\phi$, so $\exists \neg \phi$, and that every execution of $T_n'$ with $K_n \in \Optimalar$ satisfies $\phi$.

\begin{algorithm}
	\DontPrintSemicolon
	\KwData{A stit automaton $T = (Q, q_I, \Act, F, \Transrel, L, w,\lambda)$,  an obligation $A$}
	\KwResult{$\Model_T,root/h \models \Ostit{\alpha}{A}$}
	Set $root=0$\;
	Set $\Choicear=\{K \in \Act \such (q_I,K,q')\in \Delta \text{ for some }q'\} =\{K_1,\ldots,K_m\}$ \;
	\tcp*[l]{First step: find optimal actions at $root$}
	\For{$1\leq n \leq m$}{
		\tcc*[l]{Construct automaton $T_n'$ s.t. every execution of $T_n'$ is an execution of $T$ starting with action $K_n$. See Fig.~\ref{fig:algo}.}
		Create automaton $T_n$ by deleting all transitions $(q_I,K,q')$ with $K\neq K_n$ \;
		Create a copy $T_n^{\text{ren}}$ of $T_n$\;
		Create the automaton $T_n'$ as a union of $T_n^{\text{ren}}$ and $T$, with every transition $(q,K,T_n^{\text{ren}}.q_I)$ in $T_n^{\text{ren}}$ replaced by a transition $(q,K,T.q_I)$\;
		\lnl{comp vn} Compute the max value, $\vnu$, and min value, $\vnl$, of any $T_n'$ tactic starting at $q_I$\;
	}
	\tcc*[l]{An interval $ [\vnl,\vnu] $ is \textit{un-dominated} if there is no other interval $[\vnl', \vnu']$, computed in the above for-loop, s.t. $\vnl' > \vnu$}
	\lnl{undomin}Find all un-dominated intervals $[\vnl,\vnu]$\;
	\lnl{line:optimalam}Set $\Optimalar = \{K_n \in \Choicear \such [\vnl,\vnu] \text{ is un-dominated}\}$\;
	\tcc*[l]{Second step: decide whether all actions $K$  in $\Optimalar$ guarantee $A$, i.e., $K\subseteq |A|_{root}$.}
	\lnl{line:for}\For{$K_n \in \Optimalar$}{
		\uIf{$A$ is a \CTLs~formula}
		{
			\tcc{Does every execution of $T$ starting with $K_n$ satisfy $A$?}
			Use \CTLs~model-checking to check whether $T_n' \models_{\text{\CTLs}}\forall A$\;
			\If{$T_n' \not\models_{\text{\CTLs}} \forall A$ \tcp*{Optimal action $K_n$ does not guarantee $A$}}
			{\lnl{line:false 1}return False}
		}
		\lnl{line:dstit pos} \uElseIf{$A = \dstit{\alpha}{\formula}$ with $\formula\in$ \CTLs}
		{
			\tcp*[l]{This is true iff $H_{0}=|\formula|_{0}$}
			\lnl{line:all phi}Model-check whether $T\models_{\text{\CTLs}}\forall \formula$\;
			\tcc*[l]{This is true iff $K_n$ guarantees $\formula$, is not equiv. to line~\ref{line:all phi}}
			Model-check whether $T_n'\models_{\text{\CTLs}}\forall \formula$\;
			\If{$T\models_{\text{\CTLs}}\forall \formula$ or $T_n'\not \models_{\text{\CTLs}}\forall \formula$}
			{					
				\lnl{line:false 2}return False
			}
		}
		\lnl{line:dstit neg}\Else{
			\tcc{Last case: $A = \neg \dstit{\alpha}{\formula}$ with $\formula\in$ \CTLs. Similar to previous case on line \ref{line:dstit pos} with obvious modifications} 
		}
	} 
	\lnl{line:true}Return True\;
	\caption{Model checking DAU.}
	\label{algo:mc}
\end{algorithm}

\subsubsection{Computing Extremal History Utilities}
\label{sec:computing un ln}
In line~\ref{comp vn} of algorithm \ref{algo:mc}, the maximum- and minimum-valued executions of an automaton $T_n'$ must be found. This problem is related to, but distinct from, temporal logic accumulation~\cite{Boker14Accumulative} and quantitative languages~\cite{Chatterjee08Quantitative}. 
A realistic example of a $\lambda$ that can be computed is $\lambda = \min$. 
For instance, if a transition's weight $w((q, K, q'))$ is the time-to-collision when taking that transition, then the value of an execution $\lambda(\xi)$ is the shortest time-to-collision encountered along that execution. 
The best history, then, is the one with the greatest minimum time-to-collision.
To compute $\lambda(\xi)$ for $\lambda=\min$ we proceed as follows.
To avoid trivialities assume every cycle in $T_n$ is reachable.
Every infinite execution visits one or more cycles.
A \emph{simple} cycle is one that does not contain any other cycles.
A prefix is a path connecting $q_I$ to a simple cycle, and which does not itself contain a cycle.
We call an execution simple if it only loops around one simple cycle forever, possibly after traversing a prefix to get there from $q_I$.
There are finitely many simple cycles, and their prefixes are obtainable using backward reachability, so we can compute the value of every simple execution by taking the min along every connected prefix-cycle pair.
The value of a non-simple execution $\xi$ equals the value of some simple execution, since the transition of $\xi$ with minimum weight is also a transition of a simple execution, be it on a simple cycle or a prefix.
Thus, the maximum execution value $\vnu$ equals the maximum \textit{simple} execution value.
Similarly for the minimum execution value $\vnl$.

A second common accumulation function is the discounted sum function in Eq. \eqref{eq:discounted sum}. To find find the histories that carry the highest and lowest values, we cast the automaton as an extreme case of a Markov decision process (MDP).

An MDP is a control process modeled in discrete time where actions are chosen by a decision making agent, the outcomes of those actions are stochastic, and each outcome gives the agent some reward \cite{Bel}. We specify the construction of the $MDP_T$ cast from an automaton $T$ in 
the supplementary material.

Value iteration is a dynamic programming algorithm used to solve MDPs \cite{Puterman1994Mdp:}. Solving an MDP generates a policy for choosing an action at each state that optimizes some reward aggregation function $\lambda$. Following this policy from a given state $q$ (called an "optimal policy" and denoted by $\pi^*(q)$) will produce the sequence of state transitions (denoted by $\omega^*(q)$) that maximizes accumulated rewards. The expected accumulated reward for following an optimal policy from $q \in S$ is denoted by $V^*(q)$.

\begin{proposition}
	\label{prop:find extremal values}
	Given a stit automaton $T_n$,
	let $T_n^-$ be a copy of $T_n$ where the edge weights are negated,
	let $MDP_{T_n}$ be the stit MDP cast from $T_n$,
	and let $MDP_{T_n^-}$ be the stit MDP cast from $T_n^-$.
	Then, if $\lambda$ is discounted sum, the extremal values of $T_n$ are
	$\vnu = V^*(q_I)$ in $MDP_{T_n}$ and $\vnl = -V^*(q_I)$ in $MDP_{T_n^-}$
\end{proposition} 

\subsection{Model Checking Conditional Obligations}
The problem of \emph{conditional cstit model checking} is: given a stit automaton $T$ that models an agent $\alpha$, a state $q \in T.Q$, and a formula $A$ as in Section~\ref{sec:mc unconditional} (i.e. $A$ is either in \TenseLogic, or a statement of the form $\dstit{\alpha}{\formula}$ or $~\neg \dstit{\alpha}{\formula}$ where $\formula$ is in \TenseLogic), and a finite-horizon formula $B$, determine whether $\Model_{T,q},0/h \models \OstitC{\alpha}{A}{B}$ for some $h \in H_{0}$. 

$B$ is a finite horizon condition, meaning that there exists a $\tau\geq 0$ such that every history of length $\tau$ either satisfies or violates $B$. We note that if $B$ is a state formula, then either all $q$-rooted histories satisfy $B$ or none do. To avoid such trivialities, we only consider conditions that are specified by path formulae.
In this section we introduce modifications to algorithm \ref{algo:mc} and its proof (in 
the supplementary material
) that reflect this difference in determining optimal actions.
\begin{proposition}
	\label{prop:cmc is possible}
	Algorithm~\ref{alg:cmc} returns True iff $\Model,root/h\models \OstitC{\alpha}{A}{B}$. It has complexity 
	$O(\sigma(|T| + \sigma^{\tau} |T|^2 2^{|B|} + \sigma^{\tau}\cdot c_\lambda) + \sigma |T|2^{|\formula|})$, where $\sigma$ is the maximum out-degree from any state in $T$, $c_{\lambda}$ is the cost of computing the minimum and maximum values of a tactic executed on automaton $T$, $|T|$ is the number of states and transitions in $T$, $|\phi|$ is the size of the \TenseLogic~formula in A, and $|B|$ is the size of the \TenseLogic~formula for the condition.
\end{proposition}


Conceptually, getting the histories that satisfy $B$ can be done by brute force: unroll $T$'s executions up to depth $\tau$ and retain actions $K_n \in \Choicear$ that contain $B$-satisfying histories.
The values of these $B$-satisfying histories are compared to determine conditionally optimal actions, as per Def.~\ref{def:cstit} and Eq.~\eqref{eq:conditional ought}.
Once the conditionally optimal actions are determined, the algorithm continues as in Algo.~\ref{algo:mc}.

The actual model-checker constructs incrementally automata $T_{n,l}'$: every such automaton has one initial action $K_n$, has a single execution up to the horizon $\tau$, and behaves like the original automaton after $\tau$. 
Its unique execution up to $\tau$ satisfies $B$.
Algo.~\ref{alg:cmc} uses these automata to determine the conditionally optimal actions by comparing $B$-satisfying histories, in the same way that Algo.~\ref{algo:mc} uses $T_n'$ to compute (unconditionally) optimal actions.
Alg.~\ref{alg:frag} shows how to construct $T_{n,l}'$.
Each $T_{n,l}'$ has two components: a "\textit{fragment}" of $|B|$ followed by a copy of $T$. 
The \textit{fragment} is obtained by beginning with $T_n'$, removing all transitions from $q_I$ except for one $(q_I, K_n, q')$, forming the union between the resulting automaton and a copy of $T$, and checking this new automaton to see if there exists an execution that accepts $B$. 
If it does not, it aborts this branch (line~\ref{frag:continue}).
If it does, it sets $q_a = q'$ (that is, we change the state we remove transitions from) and repeats the process of removing transitions, taking the union with $T$, and checking that the automaton accepts $\exists B$. This process repeats a maximum of $\tau$ times, ensuring that the resulting automaton has a single history for $\tau$ moments, and accepts $B$. 
This final automaton is $T_{n,l}'$.

\begin{algorithm}
	\DontPrintSemicolon
	\SetKwFunction{FragStep}{fragmentStep}
	\KwData{A stit automaton $T = (Q, q_I, \Act, F, \Transrel, L, w,\lambda)$,  an obligation $A$,  a horizon-limited condition $B$, the condition's horizon $\tau$}
	\KwResult{$\Model_T,root/h \models \OstitC{\alpha}{A}{B}$}
	Set $root=0$\;
	Set $\Choicear=\{K \in \Act \such (q_I,K,q')\in \Delta \text{ for some }q'\} =\{K_1,\ldots,K_m\}$ \;
	\tcp*[l]{First step: find optimal actions at $root$}
	\For{$1\leq n \leq m$}{
		\tcc*[l]{Construct automaton $T_n'$ s.t. every execution of $T_n'$ is an execution of $T$ starting with action $K_n$. This is exactly like lines 4, 5, 6 in Algorithm \ref{algo:mc}}
		
		\tcc*[l]{Generate all automata whose first action is $K_n$ and have one history up to depth $\tau$ , that history satisfies $B$, and after that, it behaves like $T$}
	    \lnl{comp tnl}$\{T_{n,0}', \ldots, T_{n,l}'\}$ = \FragStep{$T_n', B, \tau, 1, q_I$}\; \tcp*{see Algorithm \ref{alg:frag} for \FragStep{}}
		\For{$1 \leq i \leq l$}{
		    \lnl{comp vn cond} Compute the max value, $u_{n,i}$, and min value, $\ell_{n,i}$, of any $T_{n,l}'$ tactic starting at $q_I$\;
		}
		Set $\vnu$ = $\FuncSty{max}_i(u_{n,i})$;
		Set $\vnl$ = $\FuncSty{max}_i(\ell_{n,i})$;
	}
	\lnl{undomin}Find all un-dominated intervals $[\vnl,\vnu]$\;
	\lnl{line:condoptimalam}Set $\Optimalar/B = \{K_n \in \Choicear \such [\vnl,\vnu] \text{ is un-dominated}\}$\;
	\tcc*[l]{Once all conditionally optimal actions are found, this algorithm proceeds exactly like algorithm \ref{algo:mc} starting from line \ref{undomin}}
	\caption{Conditional model checking DAU.}
	\label{alg:cmc}
\end{algorithm}

\begin{algorithm}
	\DontPrintSemicolon
	\SetKwFunction{FragStep}{fragmentStep}
	\KwData{A stit automaton $T = (Q, q_I, \Act, F, \Transrel, L, w,\lambda)$, a horizon-limited condition $B$, the condition's horizon $\tau$, the automaton depth $i$, an anchor state $q_a$}
	\KwResult{The set of stit automata that model fragments of $|B|$}
	Set $\{q_1,\ldots,q_m\}=\{q' \in Q \such (q_a,K,q')\in \Delta \text{ for some }q' \text{ and some } K\}$ \;
	\tcp*[l]{First step: find condition accepting actions at current $root$}
	\For{$1\leq l \leq m$}{
		\tcc*[l]{Construct automaton $T_l'$ s.t. every execution of $T_l'$ is an execution of $T$ starting with a transition to $q_l$.}
		Create automaton $T_l$ by deleting all transitions $(q_a,K,q)$ with $q\neq q_l$ \;
		Create a copy $T_l^{\text{ren}}$ of $T_l$\;
		Create the automaton $T_l'$ as a union of $T_l^{\text{ren}}$ and $T$, with every transition $(q,K,T_l^{\text{ren}}.q_{I,a})$ in $T_l^{\text{ren}}$ replaced by a transition $(q,K,T.q_{I,a})$ where $q_{I,a}$ is any state on an execution from $q_I$ to $q_a$\;
		\uIf{$T_l'\models \exists B$}
		{
		    \uIf{$i < \tau$}
		    {
		        Return \FragStep{$T_l'$, $B$, $\tau$, $i+1$, $q_l$};
		    }
		    \Else
		    {
		        Return $T_l'$;
		    }
		}
		\Else
		{
		    Continue;\label{frag:continue}
		}
	}

	\caption{$\FuncSty{fragmentStep}(T, B, \tau, i, q_a)$: Recursively generating fragments of $|B|$.}
	\label{alg:frag}
\end{algorithm}

\section{Case Study in Model Checking Self-Driving Cars Obligations}
\label{sec:case-study-in-model-checking-ethical-decisions-in-self-driving-cars}
As discussed in section \ref{sec:model checking}, it is common for a control engineer to model agents as an automaton, and it is natural to want to verify that the automata have some given obligations.
\yhl{The formalizations given thus far are required to reason about obligations while performing model checking and are a necessary component of our implementation}
To demonstrate the practical uses of DAU, we developed a software implementation of the model-checking algorithms of Section~\ref{sec:model checking}, and applied it to a controller for autonomous driving (adapted from \cite{GIRAULT04HybridHwyController}). We check the automaton for relevant \TenseLogic~missions, and for obligations and permissions related to the RSS rules.

\subsection{Implementation}
\label{sec:implementation}
We implemented our algorithms for model checking obligations in Python, using calls to the nuXmv symbolic model checker \cite{DBLP:conf/cav/CavadaCDGMMMRT14} to dispatch \TenseLogic model checking. Our implementation regards \yhl{Stit automata} models as directed graphs with edges labeled with action and weight. Operations on the graphs allow us to copy and take unions of automata as needed. The graphs can be translated to MDPs to find an action's extremal history utilities, or to a nuXmv model for \TenseLogic~model checking. The source code for our implementation can be found at \url{https://github.com/sabotagelab/MC-DAU}.

\subsection{Agent Model and Model Checking Results}
\label{sec:results}
\begin{figure}[t]
\centering
  \begin{subfigure}{0.5\textwidth}
    \includegraphics[width=\linewidth]{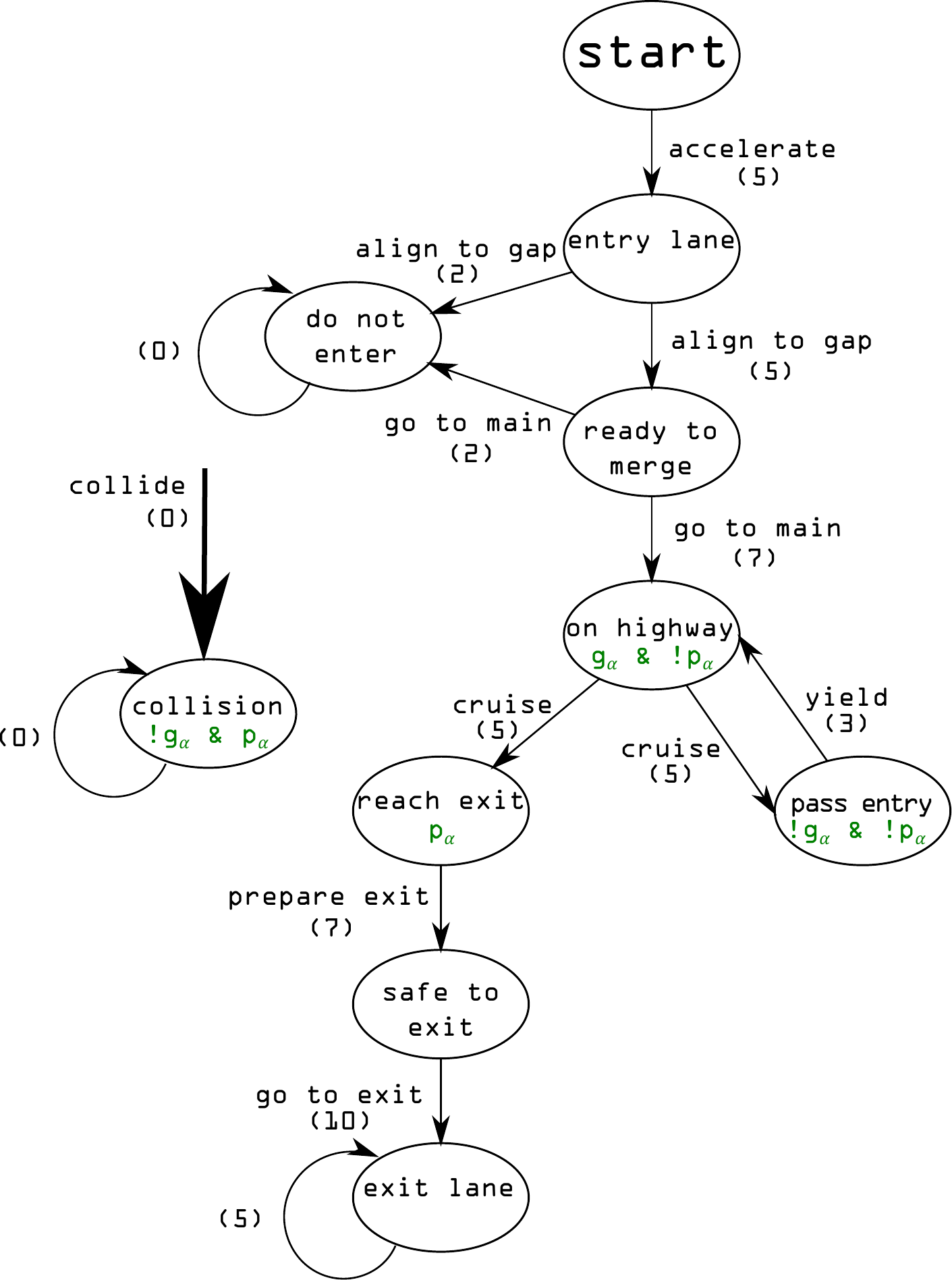}
    \caption{} \label{fig:original auto}
  \end{subfigure}%
  \hspace*{\fill}   
  \begin{subfigure}{0.5\textwidth}
    \includegraphics[width=\linewidth]{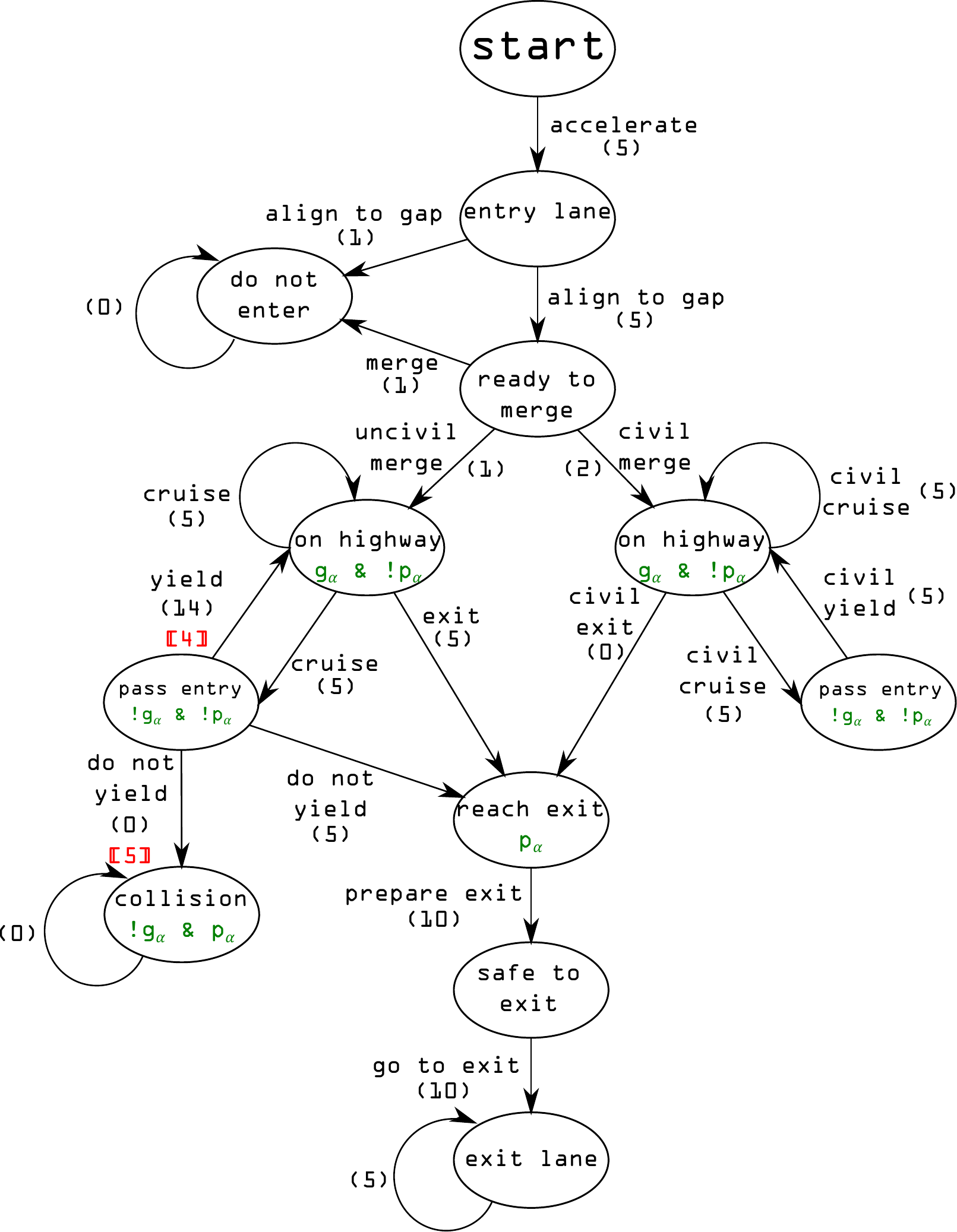}
    \caption{} \label{fig:safe redux}
  \end{subfigure}%
\caption{Highway driving agent automaton. (a) Automaton for one agent $\alpha$~\cite{GIRAULT04HybridHwyController}. Each state is labeled with the atomic propositions that hold in it, and each edge is labeled with its weight, and $\alpha$'s actions. Edges without action labels indicate loops in absorbing states. The \textit{collide} action can be taken from any state except doNotEnter, and so is denoted by the large arrow. The proposition $g_\alpha$ means the agent $\alpha$ has been given the right of way, while $p_\alpha$ means that $\alpha$ is proceeding through a conflict region. (b) Modified automaton adds an edge from \atom{passEntry} to \atom{reachExit} with action \textit{doNotYield}. It also now models implicitly a second car $\beta$, via its actions Drive and D-$\beta_{agg}$. Alternative weights are given in red brackets to make the automaton fail the ``no next collision'' obligation and gain the ``aggressive'' permission.}
\label{fig:car automaton}
\end{figure}

A hybrid continuous-time controller for autonomous highway driving is presented in~\cite{GIRAULT04HybridHwyController}.
The controller is meant to allow a car to merge onto a highway, and exit when desired.
It is shown in~\cite{GIRAULT04HybridHwyController} that if all cars are equipped with this controller, then no collisions can occur and all cars either merge and exit successfully, or drop-out, meaning that they safely abort the maneuver \yhl{and go into the doNotEnter state}.
We modeled this controller in Fig.~\ref{fig:original auto} as a stit automaton, which we will refer to as $\alpha$.
Each state is labeled with the atomic propositions that hold in it, and edges are labeled with $\alpha$'s actions, both of which are self-explanatory.
The controller's objective is to ensure safe entry, cruising, and exit; it does not determine \textit{when} to enter or exit. 
That is determined by a higher-level decision code and is captured here with atoms \atom{wantEntry} and \atom{wantExit}.
It is important to note that the collision state can be reached from almost every other state: this reflects the understanding that if  another agent, $\beta$, which is not equipped with this controller, takes a reckless action then it is impossible for $\alpha$ to avoid an accident.

We will state a number of missions, obligations, and permissions, that we might expect this automaton to satisfy, and model-check whether that is indeed the case. 
If not, we will amend the controller accordingly, thus demonstrating the value of obligation modeling and verification.

\paragraph{Missions.}
We formulate the following missions in \TenseLogic\footnote{Arguably, a logic like ATL~\cite{AlurHK02ATL} might be more appropriate here, but our focus is on DAU model-checking.}:
\begin{eqnarray}
\label{eq:car mission formulas}
\mu_1 &=& \exists \eventually (\text{\atom{onHighway}})
\\
\mu_2 &=& \exists \eventually_{\leq 4} (\text{\atom{reachExit}})
\\
\mu_3 &=& \exists \always (\neg \text{\atom{collision}})
\end{eqnarray}

The existential quantifier is used since, as noted, freedom from collision is not satisfied on all paths. 

The automaton depicted in figure \ref{fig:original auto} satisfies all the missions formulated above. The first mission ($\mu_1$) specifies that the vehicle can eventually enter the highway. 
The second ($\mu_2$) states that the vehicle can reach the exit lane within four units of time (where $\eventually_{\leq n}p$ means the proposition $p$ must be met in $n$ steps or fewer, and can be put in \ltl~syntax using Next and a bounded counter variable). 
The third mission ($\mu_3$) specifies that there is a future where the vehicle never collides. 

\paragraph{Obligations.}
For convenience, we define the \yhl{`Collision-Free'} subset of states 
$CF \defeq \{\matom{doNotEnter}\}$.

\textbf{No-collision: the role of modeling agency.}
The natural obligation $\Ostit{\alpha}{\always \neg \matom{collision}}$ is expected to fail in all states not in $CF$ since, as pointed above, there is nothing that $\alpha$ \emph{alone} can do to guarantee no collision. 
Formally, every action of $\alpha$ contains a history which satisfies $\matom{collision}$ at some moment.
The model-checker returns UNSAT in this case, as expected.
Perhaps surprisingly, the conditional obligation $\OstitC{\alpha}{\always \neg \matom{collision}}{\always \neg \matom{collision}}$ also fails in all states not in $CF$.
This obligation says that under the condition that the collision state is never visited, $\alpha$ ought to see to it that there is never a collision - which first sounds almost like a tautology.
This is where DAU's ability to model agency proves essential for a proper understanding and formalization of individual obligation.
Indeed, recall that in DAU, an agent has an obligation to ensure $A$ only if it can guarantee $A$ regardless of what other agents do (recall sure-thing reasoning and the definition of $\State_{\alpha}^m $ in Section~\ref{sec:dau}).
The condition $\always \neg \matom{collision}$ restricts our value comparisons to those actions that permit the condition to hold (Eq.~\eqref{eq:conditional ought}).
However, it is still logically false that $\alpha$ \textit{alone} can ensure no collisions: none of the conditionally optimal actions available to $\alpha$ guarantees no collisions. 
Avoidance of collisions is still a \textit{group task}, i.e. both $\alpha$ and $\beta$ must act to guarantee this - we take this up in section \ref{sec:modified auto}

\textbf{No collision next.}
Are there any states not in $CF$ at which the agent has an obligation not to collide next? To answer, we model-check the obligation
\begin{equation}
\label{eq:no next collision}
\Ostit{\alpha}{\Next \neg \matom{collision}}
\end{equation}
The model-checker confirms that this obligation can not be satisfied from any state not in $CF$. Since the agent can't guarantee another car won't collide with it, collision is included in the consequence of every action available.

\textbf{Permission vs Eventually.}
Suppose the vehicle is actually an ambulance, that occasionally has to be able to exit the highway early. 
We thus want to give it permission to exit early, without forcing that behavior, and while respecting its obligations.
So we model-check the permission 
\begin{equation}
\label{eq:permit exit 4}
\pi_1 = \Perm{\alpha}{\eventually_{\leq 4} \matom{reachExit}}
\end{equation}
from the \atom{start} state (The number 4 is rather arbitrary and is meant to suggest `early').
The model-checker informs us that the model does have this permission. 
Indeed, as long as the permission is checked from a state where \atom{reachExit} is reachable within $n$ steps, the permission 
\begin{equation}
\label{eq:permit exit n}
\Perm{\alpha}{\eventually_{\leq n} \matom{reachExit}}
\end{equation}
 will succeed for this automaton.

\textbf{Assertive vs aggressive.}
Finally, we model-check the $RSS\ref{rss-assertive}$ permission at state \atom{passEntry}.
\begin{equation}
\label{eq:permission to assert}
\pi_2 = \Perm{\alpha}{\neg \dstit{\alpha}{g_\alpha \release \neg p_\alpha}}
\end{equation}
The model-checker determines that this is satisfied.
However, we can show that this is a trivial satisfaction, which holds regardless of the weights. 
It is due to the fact that all executions of this automaton starting in \atom{passEntry} satisfy $g_\alpha \release \neg p_\alpha$.
On the other hand, consider the following aggressive DAU statement:
\begin{equation}
\label{eq:aggressive exit}
\pi_3 = \Perm{\alpha}{\dstit{\alpha}{\neg (g_\alpha \release \neg p_\alpha)}}
\end{equation}
This says that $\alpha$ is permitted to deliberately ensure that its driving is not defensive; morally, this is a less defensible permission.
It does not hold because there is no action in this automaton that guarantees $\neg(g_\alpha \release \neg p_\alpha$).

\subsection{Modified automaton.}
\label{sec:modified auto}
To draw out the effects of changing weights, we modify the automaton in Fig.~\ref{fig:original auto} to get the automaton in Fig.~\ref{fig:safe redux}, which varies in two ways. 
First, when the vehicle merges onto the highway it may choose to always yield to future traffic (by doing a `civil merge'), or to allow not yielding (by doing an `uncivil merge'). 
Second, another agent $\beta$ is modeled implicitly, removing most transitions to the collision state. This represents the agent $\beta$ avoiding collisions with agent $\alpha$ by taking a \textit{drive} action. The remaining transition to collision is taken when $\alpha$ chooses the \textit{do not yield} action and $\beta$ chooses a determined $\beta$ aggression (or \textit{D-$\beta_{agg}$}) action. We confirmed that this automaton still satisfies the mission formulae $\mu_1$, $\mu_2$, and $\mu_3$.

\textbf{No collision}.
With these changes, we can revisit the problem of specifying an obligation to not collide. While the obligation $\Ostit{\alpha}{\always \neg \matom{collision}}$ still fails from \atom{start}, it holds (though trivially) from the many states that no longer have a path to \atom{collision}. On the other hand, the obligation $\Ostit{\alpha}{\Next \neg \matom{collision}}$ in equation \eqref{eq:no next collision} non-trivially holds from the \atom{passEntry} state adjacent to \atom{collision}. 
This is ensured by weighting the \textit{yield} transition relatively heavily --- guaranteeing that the \textit{yield} action is the only optimal action.

\textbf{Permission vs Eventually}.
Suppose again this is an ambulance that occasionally needs to exit the highway early. 
The permission $\Perm{\alpha}{\eventually_{\leq n} \matom{reachExit}}$ no longer necessarily holds in states where \atom{reachExit} is reachable within $n$ steps. 
We demonstrate this from the \atom{onHighway} state reached by the \textit{civil merge} action. 
By making \textit{civil cruise} the optimal action, we guarantee that the optimal histories spend at least one moment in \atom{onHighway} before moving to \atom{reachExit}. 
This yields an ethically difficult position where an insistence on defensive driving negates the permission to exit the highway early, though it might be needed.

\textbf{Assertive vs. Aggressive.}
Finally, we model-check again the RSS-type permission in Eq.~\eqref{eq:permission to assert} from \atom{passEntry}.
This does \textit{not} hold in this model, as determined by the model-checker.
Similarly, the permission in Eq.~\eqref{eq:aggressive exit} does \textit{not} hold because no optimal action guarantees $\neg(g_\alpha \release \neg p_\alpha)$.

However, by changing the weights of this automaton as depicted by the red, bracketed weights in Fig. \ref{fig:safe redux}, we can satisfy permissions $\pi_2$ and $\pi_3$ at the cost of the ``no next collision'' obligation in Eq. \eqref{eq:no next collision}. 
By ensuring that \textit{do not yield} is an optimal action, we know that not all optimal actions guarantee $g_\alpha \release \neg p_\alpha$ (thus $\pi_2$ is satisfied), and we know that there exists an optimal action that guarantees $\neg(g_\alpha \release \neg p_\alpha)$ (thus $\pi_3$ is satisfied). As a consequence of \textit{do not yield} being counted as an optimal action, the ``no collision next'' obligation fails.


\section{Related Work}
\label{sec:related work}
\textbf{Deontic logic and autonomous systems ethics.}
The need to encode and study ethical and social obligations for human-scale CPS is well-recognized~\cite{Lim14RobotEthics,thekkilakattil15ethicsCPS,KulickiTMDeon18}, though little explored technically.
This paper follows the logicist program~\cite{Arkoudas06LogicistEthics} in approaching this problem, within which the deontic family of logics takes pride of place having been created specifically to reason about obligations.
Standard Deontic Logic has many well-known paradoxes~\cite{McNamaraChapter}, which have spurred the proposal of alternatives to remedy them~\cite{DLHandbook}. 
Some variations are used to specify legal and software contracts as in~\cite{Prisacariu2012CL}.
Alternating-time Temporal Logic (ATL) was proposed in~\cite{AlurHK02ATL} to reason about groups of agents, and used in~\cite{Broersen06StrategicDL} to reason about strategic obligations, and it will be interesting to connect the modeling of agency between DAU and ATL. 
Finally, RSS rules have been encoded in Signal Temporal Logic for the purpose of monitoring them over linear traces in~\cite{hek_rss-stl}, but notions of obligation and uncertainty were not investigated.

\textbf{Temporal propagation.}
The most relevant work on the propagation of obligations is ~\cite{Brunel08Propagation}, which takes a near-product of Standard Deontic Logic and LTL to study propagation, and ends up with a semantics that resembles DAU (albeit LTL is linear time). 
Works that integrate deontic and temporal modalities more generally include~\cite{Giordano13TDLASP} (to specify business processes),~\cite{Raimondi04automaticverification} for interpreted systems, and~\cite{agotnes09NTL} for contextualized (normed) obligations.

\textbf{Algorithmic aspects.}
Most of the work in deontic logic has been concerned with finding the `right' axioms and inference rules that formalize our intuition about obligations and permissions, with algorithmic aspects receiving comparatively little attention. 
Broader work in normative multi-agent systems relies on simulation to study, for example, ways in which social norms arise~\cite{Boella06IntroToNORMAS}.
Decision procedures exist for some logics, 
like the KED theorem prover for Standard Deontic Logic~\cite{Artosi94ked},
and the decision procedures in~\cite{BALBIANI09Decision}.
There are even fewer implemented tools, such as MCMAS, the OBDD-based checker in~\cite{Lomuscio2017} for the logic of~\cite{Raimondi04automaticverification},
and the implementation of dyadic deontic logic in Isabelle/HOL in~\cite{BenzmullerDeon18DyadicHOL}.
A proof system for a simplified version of DAU has been developed in~\cite{Murakami04utilitariandeontic,Arkoudas05towardethical} to determine whether certain obligations follow from others (a `trusted base').
We propose a model-checker, to determine whether a given automaton has an obligation, by examining directly the values it assigns to its executions.
In a deployed system, theorem-proving and model-checking are likely to play complementary roles.

\textbf{Interpretability.} 
In DAU, an agent that always performs optimal actions is one that always meets its obligations. 
Therefore, DAU can be viewed, informally, as the logic of utility maximization.
As such, it gives a \textit{logical interpretation} to the behavior of controlled systems that maximize long-term utility, such as~\cite{Gerdes2015}.
This connects DAU to the field of interpretable AI~\cite{Jha19SparseBoolean}, albeit from a non-statistical perspective.

\section{Conclusions}
\label{sec:conclusions}
We have discussed and demonstrated the use of Dominance Act Utilitarian deontic logic for the formalization of obligations and permissions for autonomous systems. We investigated the interaction of temporal and deontic modalities to find patterns for temporal propagation of obligations. We expressed self-driving car obligations from RSS in DAU, and found undesirable consequences of these norms. We introduced algorithms to allow system designers to automatically determine if a system has an obligation, and demonstrated an implementation of these algorithms.

In the pursuit of an algorithmic account of a system's obligations, it would be desirable next to synthesize given obligations by automatically adjusting the weights. DAU could also be used in tandem with inverse reinforcement learning to learn the obligations of an agent by observing its behavior. It will also be important to study the inheritance of obligations between groups and individuals, i.e. knowing how the obligation of a group of agents impacts the obligations of agents in that group. Since deontic logic was designed for the study of ethics, this work opens the way for formal ethical analysis of autonomous system design.
These considerations will help determine the suitability of DAU, and deontic logic more generally, for the design and verification of autonomous systems.

\bibliographystyle{ACM-Reference-Format}
\bibliography{iccps2017,hscc17,hscc2016,fainekos_bibrefs,hscc19,cav2019,colin_bib}
%

\end{document}